\documentclass{article} 
\usepackage{iclr2027_conference,times}

\usepackage{amsmath,amsfonts,bm}

\def\eqref#1{equation~\ref{#1}}

\def\1{\bm{1}}

\DeclareMathAlphabet{\mathsfit}{\encodingdefault}{\sfdefault}{m}{sl}
\SetMathAlphabet{\mathsfit}{bold}{\encodingdefault}{\sfdefault}{bx}{n}

\usepackage{hyperref}
\usepackage{url}
\usepackage{amssymb}
\usepackage{booktabs}
\usepackage{amsthm}
\newtheorem{theorem}{Theorem}
\newtheorem{proposition}{Proposition}
\newtheorem{corollary}{Corollary}

\usepackage{algorithm}
\usepackage{algpseudocode}
\usepackage{xcolor}
\usepackage{graphicx}
\usepackage{subcaption}
\usepackage{makecell}
\usepackage{array}
\usepackage{comment}
\usepackage{enumitem}
\usepackage{wrapfig}
\usepackage{caption}
\usepackage{multirow}
\newcommand{\topic}[1]{\noindent\textbf{#1}}
\newcommand{\ourmethod}{TabMosaic}

\title{%
    \ourmethod{}: Unified In-Context Diffusion \\for Tabular Generative Tasks}

\author{
Jayoung Kim\textsuperscript{1}, Sanghyun Hong\textsuperscript{2}, Noseong Park\textsuperscript{1}\\
\textsuperscript{1}KAIST, \textsuperscript{2}Oregon State University\\
\texttt{\{jayoung.kim, noseong\}@kaist.ac.kr, sanghyun.hong@oregonstate.edu}
}
\iclrfinalcopy 
\begin{document}
\addtocontents{toc}{\protect\setcounter{tocdepth}{-1}}

\maketitle

\begin{abstract}
Generative modeling for tabular data 
has advanced to the point where models can synthesize
realistic records, recover missing values, and generate samples 
conditioned on class labels or observed attributes.
Despite recent advances, existing approaches 
either require table-specific training or
restrict in-context learning to single-cell prediction.
In this work, we formulate these tabular generative tasks as a single problem:
modeling the joint conditional distribution of an arbitrary set of target cells
given the observed cells of an unseen table.
To this end, we turn to diffusion, 
whose denoising objective is defined jointly 
over all target cells rather than one cell at a time.
Under the pretraining distribution, this objective is minimized by 
the posterior predictive distribution over the target set.
Building on this, we propose \ourmethod{}, 
a unified in-context diffusion model for generating mixed-type target cells 
in unseen, potentially incomplete tables without table-specific training.
We further show that its pretraining objective 
bounds the conditional negative log-likelihood 
and the divergence from the posterior predictive for arbitrary target sets,
extending the Bayesian account of in-context learning 
from a single-value prediction to joint multi-cell generation.
Across 48 real-world tables, \ourmethod{} 
ranks first on all evaluated generative tasks 
against 20 baseline approaches.
%
\end{abstract}

\section{Introduction}
\label{sec:intro}

%
In domains such as healthcare, finance and scientific discovery,
tabular data is widely used for data-driven decision-making
~\citep{borisov2022deep, van2024tabular}.
Recent 
tabular generation models synthesize 
realistic rows, recover missing values, 
and generate samples conditioned on classes or attributes%
~\citep{kotelnikov2023tabddpm, kim2023stasy, zhang2024tabsyn, 
        shi2025tabdiff, zhang2025tabnat, jarrett2022hyperimpute, 
        du2024remasker, zhang2025diffputer}.
Yet these models are typically dataset-specific:
they must be trained or fine-tuned separately for each new table,
requiring a new training run whenever the dataset changes.
This motivates a natural question:

\begin{center}
\emph{Can a single pretrained model perform
      diverse tabular generative tasks on unseen tables \\ 
      without table-specific training, using only in-context information?}
\end{center}

Recent work has been moving toward this goal
by transferring tabular models across datasets.
Cross-table generative models pretrain across many tables,
reducing the need to learn each dataset entirely from scratch,
but still require fine-tuning on each new one%
~\citep{lin2025ctsyn, jiang2026tabular}.
In-context tabular models remove this adaptation step
and operate directly on unseen tables%
~\citep{hollmann2025accurate, qu2025tabicl}.
However, these models are trained for individual prediction targets
rather than the joint distribution over an arbitrary set of cells, 
so multi-cell generation and imputation must be assembled from 
separate predictions rather than modeled directly.
No existing work provides a single pretrained model 
that learns the joint conditional distribution over arbitrary target sets 
and applies it directly to unseen tables without table-specific adaptation.

\begin{figure}[t]
\centering
\includegraphics[width=0.9\linewidth]{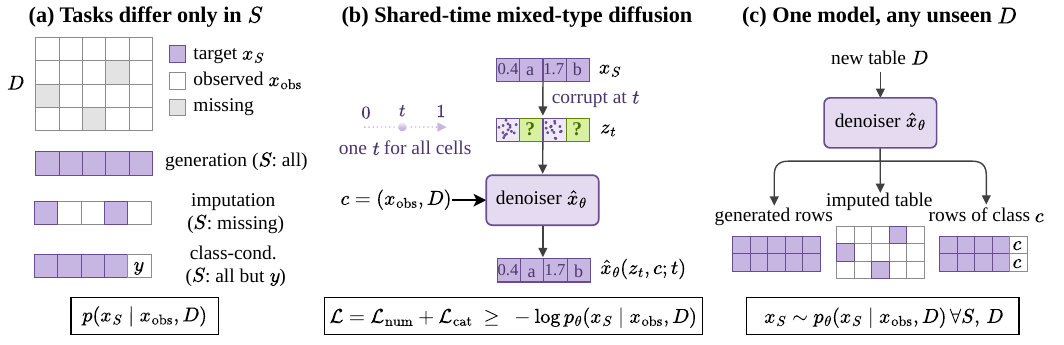}
\caption{\textbf{Overview of our proposed framework.}
(a) A unified formulation expresses generation, imputation 
and class-conditional generation as the same conditional query,
only differing in the target set $S$.
(b) A shared-time diffusion objective 
jointly models mixed numerical and categorical targets.
(c) Our \ourmethod{} applies directly 
to unseen tables without table-specific training.}
\label{fig:overview}
\vspace{-1.5em}
\end{figure}

\topic{Our contributions.}
In this work, we formulate tabular generation, imputation 
and conditional generation as a single problem:
learning the joint conditional distribution 
over arbitrary target sets on unseen tables.
This formulation shifts tabular generative modeling 
from training a new model for each dataset and task 
to learning a reusable conditional model once across datasets.
Different generative tasks are specified simply 
by which cells are observed and which are queried.

To this end, we \emph{first} formulate 
a unified learning objective for tabular generative tasks.
Generation, imputation, and conditional generation 
are treated as the same family of conditionals,
differing only in which cells are designated as targets.
We instantiate this objective 
with in-context diffusion over arbitrary sets of
mixed numerical and categorical target cells,
without imposing an autoregressive order.
We show that, under the pretraining prior, 
its population optimum is the posterior predictive distribution.
We derive a variational upper bound 
on the conditional negative log-likelihood.
Under the pretraining prior,
the bound also controls the KL divergence 
from the posterior predictive distribution 
to the model distribution, up to a model-independent constant.

\emph{Second}, we introduce \ourmethod{},
an in-context diffusion model 
that realizes this objective on unseen tables.
The model jointly captures dependencies 
among numerical and categorical cells 
while conditioning on the observed cells and context.
To scale \ourmethod{} to long and incomplete tables,
we develop a two-staged pretraining strategy:
the model first learns from short, complete tables and 
is then extended to long, partially observed contexts.
In the second stage, multiple corruptions share a single context encoding,
reducing computation and gradient variance.
Missing context cells are represented explicitly,
so the model can condition directly on incomplete tables at inference time.

\emph{Third}, we evaluate \ourmethod{} on 
generation, imputation, and class-conditional generation across the 48 datasets 
from a recent tabular benchmark~\citep{erickson2026tabarena}.
Against 20 baselines spanning dataset-specific generators and imputers,
in-context predictors, and training-free conditional samplers,
our single pretrained \ourmethod{} \emph{ranks first} across all three tasks.
We also compare \ourmethod{} against 
the exact posterior predictive distribution 
on synthetic priors for which the posterior is available in closed form.
For imputation, we evaluate performance 
under different missingness patterns and mechanisms, 
such as settings where entire columns are missing.
We conduct an ablation study 
to assess the contribution of the staged pretraining strategy.

\section{Background and Related Work}
\label{sec:related}


\subsection{In-context learning as posterior prediction}
\label{sec:related-icl}

A sequence model trained by maximum likelihood on data drawn from a mixture over latent tasks
approximates, at its optimum, the Bayesian posterior predictive under that mixture: conditioning
on a prompt plays the role of conditioning on observations, and the model marginalizes the latent
task rather than inferring it explicitly. This is the reading of in-context learning as implicit Bayesian inference~\citep{xie2021explanation}, and the
basis of a line of work that trains sequence models on synthetic priors and compares them
with the Bayes-optimal predictor on tasks whose posterior is available in closed
form~\citep{garg2022can, raventos2023pretraining, panwar2024context}. Tables fit this view, since the rows of a table are exchangeable draws from an unknown mechanism, and prior-data fitted networks~\citep{muller2022transformers} make it concrete for one cell: a transformer pretrained on synthetic tables drawn from a prior $p(\phi)$ over data-generating mechanisms approximates, for a query row with features $x$, label $y$, and context rows $D$, the posterior predictive distribution under the prior:
\begin{equation}
    p(y \mid x, D) \;=\; \int p(y \mid x, \phi)\; p(\phi \mid D)\; d\phi
\label{eq:ppd-single}
\end{equation}
The same argument applies to any target set $x_S$, provided that the pretraining objective is the likelihood of $x_S$ itself.

\vspace{-0.5em}

\subsection{Generative models for tables}
\label{sec:related-diffusion}

\topic{Per-dataset models.}
Diffusion models pair a fixed corruption process with a learned reverse process, Gaussian for continuous data~\citep{ho2020denoising, song2021score, karras2022elucidating} and masked for discrete data~\citep{austin2021structured, sahoo2024simple}; GIDD~\citep{vonrutte2025generalized} interpolates the mask with uniform noise, which yields a self-correcting sampler with an exact ELBO, and we use the EDM~\citep{karras2022elucidating} and GIDD kernels (\S\ref{sec:method-model}).
For tables, generators of this kind are trained per dataset, in score-based, latent, mixed-type continuous-time diffusion model forms~\citep{kim2023stasy, zhang2024tabsyn, shi2025tabdiff} or with gradient-boosted trees in place of a network~\citep{jolicoeur2024generating}; TabNAT pairs a masked transformer for discrete columns with a diffusion loss for continuous ones~\citep{zhang2025tabnat}.
Per-dataset imputers follow the same pattern~\citep{van1999flexible, stekhoven2012missforest, yoon2018gain, jarrett2022hyperimpute, du2024remasker, zhang2025diffputer}.
Each recovers the conditionals over a target set in one of two ways:
i) Unconditional models are conditioned after training by inpainting, which fixes the given cells at every reverse step; the result approximates the conditional of the trained joint rather than sampling it, and the released procedures cover a single target column or cells of one type~\citep{zhang2024tabsyn, shi2025tabdiff, jolicoeur2024generating}; TabDDPM instead conditions on one column designated before training, the class label~\citep{kotelnikov2023tabddpm}.
ii) TabNAT~\citep{zhang2025tabnat} and GReaT~\cite{borisov2022language} train conditionals of a random target set given the rest by autoregression over a random order, which imposes an order on unordered set of columns; each order is a different factorization, and no term of the objective makes them agree.
Both are trained on a complete table, so the context itself cannot contain missing cells.

\topic{Pretraining across tables.}
CTSyn~\citep{lin2025ctsyn} pretrains an autoencoder into a shared latent space and a schema-conditioned latent diffusion model across tables, and fine-tunes both on each new table; TabFORGE~\citep{jiang2026tabular} likewise fine-tunes per table.
Both target unconditional generation, and neither answers a query on an unseen table without a gradient update.

\topic{In-context predictors.}
Models pretrained to predict single cells also impute and generate without training on the table: missing cells are filled by chaining a point predictor over them, and rows are generated by chaining the predictor over the columns, optionally followed by rounds of masking and refilling~\citep{hollmann2025accurate, zhang2025limix}.
The objective is a sum of per-cell log-likelihoods given the unmasked ones, so no term involves a target set as a whole, and a generated row is a chain in a chosen order.

\topic{Energy samplers on a classifier.}
TabEBM~\citep{margeloiu2024tabebm} and TabPFGen~\citep{ma2024tabpfgen} use the logits of a pretrained tabular classifier as an energy and draw class-conditional rows by Langevin dynamics; TabPFGen also imputes by running the same Langevin dynamics on the missing cells only, with the observes cells held fixed and the label given.
The classification objective constrains only the class posterior and leaves the distribution over the cells free~\citep{grathwohl2019your}, every query must supply a label, and the model has no likelihood.

\vspace{-0.5em}
\section{A Unified Objective for Tabular Generative Tasks}
\label{sec:method}

\topic{Setting.}
A tabular task instance consists of a \emph{context} set $D = \{r_1, \dots, r_n\}$ of $n$ rows and a query row $x = (x_1, \dots, x_F)$ over the same $F$ columns, where each $r_i \in \mathbb{R}^F$ is a row of the table whose cells may be missing, and the cells of $x$ split into an observed part
$x_{\mathrm{obs}}$ and a \emph{target set} $S\subseteq \{1, \cdots, F\}$ of cell indices of the row, $S = S_{\mathrm{num}} \cup S_{\mathrm{cat}}$ of numerical and categorical indices. Multiple query rows are handled identically and independently
given $D$, and therefore we describe a single row. We write $c = (x_{\mathrm{obs}}, D)$ for the
conditioning information. Note that the tasks we consider are instances of a
\emph{single} distributional object, the conditional joint distribution over the target cells,
and they differ only in the choice of $S$, as shown in Fig.~\ref{fig:overview} (a).



\subsection{One family of tasks: conditional queries over cell sets}
\label{sec:method-tasks}

We pose all of the tabular generative tasks above as learning the family
$\{\, p(x_S \mid x_{\mathrm{obs}}, D) \,\}_{S}$ with one model. 
Under a prior $p(\phi)$ over data-generating mechanisms, 
where the rows of $D$ and the query row are drawn i.i.d.\ from
$p(\cdot \mid \phi)$, the ideal member of this family is the 
posterior predictive distribution of Eq.~\ref{eq:ppd-single},
generalized from a single value to an arbitrary cell set:
\begin{equation}
p(x_S \mid x_{\mathrm{obs}}, D)
\;=\; \int p(x_S \mid x_{\mathrm{obs}}, \phi)\; p(\phi \mid D)\; d\phi.
\label{eq:ppd}
\end{equation}
Learning this family imposes three requirements, which drive our design: 
i) a valid \emph{joint} distribution for every mixed-type target set $S$; ii) a \emph{single} pretrained network that
serves every $(S, x_{\mathrm{obs}})$ in-context; and iii) a training objective that provably targets Eq.~\ref{eq:ppd}.

\subsection{In-context joint diffusion over cell sets}
\label{sec:method-model}

\paragraph{Forward process.}
The diffusion state at time $t \in [0,1]$ is
$z_t = (z_t^{\mathrm{num}}, z_t^{\mathrm{cat}}) \in \mathbb{R}^{|S_{\mathrm{num}}|} \times
\prod_{j \in S_{\mathrm{cat}}} \mathcal{V}_j$, 
where $\mathcal{V}_j$ is the class set of column $j$ 
augmented with a \textsc{mask} token. We corrupt only the target cells; 
the observed cells and the context remain clean. 
Given the clean values $x_S$, the forward marginal factorizes over
cells on a \emph{shared} time axis, i.e., a Gaussian (EDM) kernel on 
numerical cells and a GIDD kernel on categorical cells:
\begin{equation}
q_t(z_t \mid x_S)
= \prod_{j \in S_{\mathrm{num}}} \mathcal{N}\,\big(z_{t,j};\, x_j,\, \sigma(t)^2\big)
\;\cdot \prod_{j \in S_{\mathrm{cat}}} \mathrm{Cat}\,\big(z_{t,j};\, \alpha_t e_{x_j} + \beta_t \pi_t\big),
\label{eq:forward}
\end{equation}
where $\sigma(t)$ is the EDM noise schedule, $e_{x_j}\in \mathbb{R}^{|\mathcal{V}_j|}$ 
is the one-hot vector of the clean class $x_j$, which is never \textsc{mask}, 
$\beta_t = 1-\alpha_t$, and $\pi_t$ is the mixing distribution of GIDD~\citep{vonrutte2025generalized}, a mixture of \textsc{mask} and uniform noise whose uniform component lets the reverse process revise cells that have already been unmasked (See Appendix~\ref{app:kernels-gidd}).

\paragraph{Reverse process.}
A single denoiser $\hat{x}_\theta(z_t, c; t)$ reads the \emph{entire} mixed state together with
the conditioning $c=(x_\mathrm{obs}, D)$, and outputs a scalar estimate for each numerical target cell and a class distribution for each categorical target cell. On a time grid $1 = t_N > \dots > t_0 = 0$, each reverse step is the posterior of the forward process with the clean target value replaced by the denoiser output, and every factor is conditioned on the full mixed state:
\begin{equation}
p_\theta(z_s \mid z_t, c)
= \prod_{j \in S} q\big(z_{s,j} \mid z_{t,j},\, \hat{x}_{\theta,j}(z_t, c; t)\big).
\label{eq:step}
\end{equation}

The model is the composition of the reverse steps, started from the terminal distribution $p(z_{t_N})$, which is independent of the data (Gaussian noise on numerical cells, \textsc{MASK} on categorical cells):
\begin{equation}
p_\theta(x_S \mid c)
= \int p(z_{t_N}) \prod_{i=N}^{1} p_\theta\big(z_{t_{i-1}} \mid z_{t_i}, c\big)\, dz_{t_1:t_N}.
\label{eq:model}
\end{equation}

\begin{proposition}[Validity]
\label{prop:valid}
For every $S$, $c$, and $\theta$, Eq.~\ref{eq:model} defines a normalized probability distribution on the mixed product space. The proof is in Appendix~\ref{app:proofs}.
\end{proposition}

We highlight three properties of the proposed construction. i) The joint is factorized along the \emph{time axis}, not along a cell or modality ordering, so no ordering is chosen and the two modalities are coherent by construction (Proposition~\ref{prop:valid}). ii) Dependencies between cells, including across modalities, arise from conditioning: each factor of Eq.~\ref{eq:step} reads the current state of every other cell, and the composition over steps transports this information. iii) With a single step ($N{=}1$) the model reduces to the conditionally independent product $\prod_{j \in S} p_\theta(x_j \mid c)$; the iteration is what creates the coupling.

\subsection{Training objective and guarantee}
\label{sec:method-objective}

\paragraph{Target patterns.}
During pretraining the target pattern is drawn independently of the cell values and of $t$: with probability $0.4$ all query cells are targets; otherwise each cell is a target independently with rate $\rho \sim \mathcal{U}(0,1)$. A single $t$ then corrupts all target
cells. Drawing $S$ and $t$ independently is valid because the bound of Theorem~\ref{thm:elbo} holds for every target
set separately, i.e., $S$ only selects which cells enter it and $t$ ranges over $[0, 1]$ for each $S$.
As a result, every combination of target set and noise level lies inside the training distribution, so a query can start its reverse process at any $t$ and for any fraction of target cells.

\paragraph{Loss.}
Let $\lambda(\sigma)$ be the EDM weight and $w(z_{t,j})$ be the state-dependent GIDD weights.
Our pretraining loss is defined as follows:
\begin{equation}
\mathcal{L} = \mathbb{E}_{(x \cup D) \sim p(\mathcal{D})}\,
\mathbb{E}_{S}\, \mathbb{E}_{t}\, \mathbb{E}_{z_t \sim q_t}
\big[\, \tfrac{1}{|S|}\left( \mathcal{L}_{\mathrm{num}} + \mathcal{L}_{\mathrm{cat}}\right)  \,\big],
\label{eq:loss}
\end{equation}
\begin{equation}
\mathcal{L}_{\mathrm{num}}
= \!\sum_{j \in S_{\mathrm{num}}}\!\! \lambda(\sigma(t))
\big(\hat{x}_{\theta,j} - x_j\big)^2,
\qquad
\mathcal{L}_{\mathrm{cat}}
= \!\sum_{j \in S_{\mathrm{cat}}}\!\! w(z_{t,j})
\big[D_{\mathrm{KL}, j} + D_{\mathrm{IS},j}\big],
\label{eq:loss-terms}
\end{equation}
where $D_{\mathrm{KL},j}$ and $D_{\mathrm{IS},j}$ are the divergence terms of the GIDD kernel. We
emphasize that the per-cell sums are \emph{not} a factorization of the model: they arise since
the per-step distributions in Eq.~\ref{eq:forward} and Eq.~\ref{eq:step} factorize given $z_t$,
whereas the joint distribution in Eq.~\ref{eq:model} does not.

\begin{theorem}[Conditional ELBO]
\label{thm:elbo}
For fixed $(x, S, c)$, let $\ell^\star_\theta(x, S, c)$ be the inner expectation
$\mathbb{E}_t\,\mathbb{E}_{z_t}[\cdot]$ of Eq.~\ref{eq:loss} without the normalization by
$|S|$ and with the numerical weight $\lambda$ replaced by the likelihood weight
$w_{\mathrm{lik}}(t) = \sigma'(t)/\sigma(t)^3$; the categorical weights are already those of the
likelihood bound. Then $\ell^\star_\theta(x, S, c) \ge -\log p_\theta(x_S \mid c) - C$, where
$C$ is independent of $\theta$, i.e., $-\ell^\star_\theta - C$ is an evidence lower bound on
$\log p_\theta(x_S \mid c)$. The proof is in Appendix~\ref{app:proofs}.
\end{theorem}

\begin{corollary}[The objective targets the posterior predictive distribution]
\label{cor:ppd}
Let tables be drawn from the prior and $S$ independently of the cell values. Then
\begin{equation*}
\mathbb{E}\big[\mathrm{KL}\big(p(x_S \mid x_{\mathrm{obs}}, D)\,\big\|\,
p_\theta(x_S \mid x_{\mathrm{obs}}, D)\big)\big]
\;\le\; \mathbb{E}\big[\ell^\star_\theta\big] - C',
\end{equation*}
where $p(x_S \mid x_{\mathrm{obs}}, D)$ is the posterior predictive of Eq.~\ref{eq:ppd} and $C'$ is
independent of $\theta$. Moreover, $\mathcal{L}$ and $\mathbb{E}[\ell^\star_\theta]$ are
minimized by the same denoiser: $\hat{x}_{\theta,j} = \mathbb{E}[x_j \mid z_t, c]$ for numerical
cells and the posterior PMF $p(x_j \mid z_t, c)$ for categorical cells. The proof is in
Appendix~\ref{app:proofs}.
\end{corollary}

\begin{corollary}[The posterior predictive minimizes the objective]
\label{cor:min}
Under the assumptions of Corollary~\ref{cor:ppd}, the denoiser of the posterior predictive of Eq.~\ref{eq:ppd}, i.e., $\hat{x}^\star_j = \mathbb{E}[x_j \mid z_t,c\,]$ for numerical and $p(x_j \mid z_t, c\,)$ for categorical cells, minimizes $\mathcal{L}$ and $\mathbb{E}[\ell^\star_\theta]$, and with it the model of Eq.~\ref{eq:model} equals the posterior predictive as $N \to \infty$. The proof is in Appendix~\ref{app:proofs}.
\end{corollary}

Corollary~\ref{cor:ppd} bounds the gap to the posterior predictive by the objective, and Corollary~\ref{cor:min} shows that the posterior predictive attains it, i.e., it is the minimizer and not only the target. Together they make the Bayesian reading of in-context learning exact for our
objective, with the cross-entropy identity relaxed to an upper bound through the diffusion ELBO. The condition on $S$
holds by construction, since every target pattern we draw is independent of the cell values, which makes the target the posterior predictive rather than a conditional that also depends 
on which cells are hidden.

\section{\ourmethod{}: A Unified In-Context Diffusion Model}
\label{sec:training-and-inference}

\subsection{Architecture}
\label{sec:method-arch}

The denoiser operates at the resolution of individual cells, so that every cell of the context and the query can be a target, an observation, or missing. Query rows attend to the context only, so they are answered independently and in parallel, as Eq.~\ref{eq:ppd} is defined for one row at a time.

\paragraph{Tokenization.}
Each cell of the context and the query row is one token: a categorical cell by a lookup of its class (or \textsc{mask}), a numerical cell by a linear map of its value. A learned \emph{observation-state} embedding marks the cell as a target (a cell of $x_S$), observed (a cell of $x_{\mathrm{obs}}$ or of $D$), or missing in the context (a placeholder value with a missing flag), so the target set can differ across rows within one forward pass. The time embedding is added to target tokens only, so the context representation does not depend on $t$. 
Consequently, the keys and values of the context tokens are reused across denoising steps (\S\ref{sec:method-inference}) and the multiple corruptions (\S\ref{sec:method-pretraining}).  No token carries a column index or a role such as ``label'': the label is an ordinary column.

\paragraph{Backbone.}
We follow the two-axis attention of in-context tabular transformers~\citep{hollmann2025accurate}, which alternates between the cells of a row and the rows of a column. We differ in the following aspects:
i) Tokens carry no column-specific embedding, neither random per-column embeddings nor positional encodings over columns, so the network is permutation-equivariant in the columns.
ii) Instead of a single label head, a per-cell head, shared by all cells, returns a scalar for each numerical target and class logits for each categorical target.

\subsection{Pretraining: prior, estimator, and stages}
\label{sec:method-pretraining}

\paragraph{Synthetic prior.}
We draw tables from a structural causal model (SCM) prior of TabICL (v2)~\citep{qu2026tabiclv2} with two changes that follow from \S\ref{sec:method-tasks}: a label is an ordinary column, whose cardinality follows the same sampling scheme as the categorical features, and half of the tables have a numerical label. 

\paragraph{Multi-corruption estimator.}
The loss of Eq.~\ref{eq:loss} is a nested expectation: over tables and target sets, and,
inside, over the corruption $z_t \sim q_t(\cdot \mid x_S)$. One corruption per table estimates the inner expectation without bias but with noise. To reduce this noise, for each table we draw one pair $(t, S)$ and $K$ independent corruptions $z_t^{(1)}, \dots, z_t^{(K)} \sim q_t(\cdot \mid x_S)$, and average
their losses,
\begin{equation}
\hat{\ell}_K(x, S, c) = \frac{1}{K} \sum_{k=1}^{K} \frac{1}{|S|}
\big[\mathcal{L}_{\mathrm{num}} + \mathcal{L}_{\mathrm{cat}}\big]\big(z_t^{(k)}\big),
\label{eq:multi}
\end{equation}
which is unbiased for the inner expectation and divides the variance due to the corruption by $K$. The cost is far below $K$ times that of one corruption: since query rows attend only to the context, the $K$ copies are packed along the row axis and share one context encoding, which is most of the cost of a step in Stage~2, so only the query-side compute grows with $K$. Averaging the outputs instead of the losses would change the objective (See Appendix~\ref{app:ensemble}).


\paragraph{Two stages.}

Pretraining proceeds in two warm-started stages, each widening the training distribution and adding a capability. $T$ is the number of rows per table, $F$ is the number of features. 
\begin{enumerate}\setlength{\itemsep}{2pt}\setlength{\topsep}{2pt}
    \item \textbf{Stage~1: joint on short, complete tables.} Each table has $T = 1{,}024$ rows and $F \sim \log\mathcal{U}(1, 100)$ features, and the context fraction is sampled from $\mathcal{U}(0.3, 0.9)$. Tables have no missing values. We train on $32$M tables with a learning rate of $8 \times 10^{-4}$. Multi-corruption is not activated ($K = 1$).
    
    \item \textbf{Stage~2: long and incomplete contexts.} We sample $T \sim \log\mathcal{U}(400, 10{,}240)$ rows and then $F \sim \log\mathcal{U}(4, \min(200, B/T))$ features under a cell budget $B = 250,000$, since the cost of the denoiser grows with $T \times F$. The context fraction is fixed to $0.8$. Context cells are hidden under MCAR, MAR, or MNAR at a per-table rate $\rho \sim \mathcal{U}(0, \rho_{\max})$, where $\rho_{\max}$ increases linearly from $0$ to $0.2$ over the first quarter of the stage. We train on $2.56$M tables with a learning rate of $10^{-4}$ and $K = 4$.    
    
    \end{enumerate}
Neither stage changes the minimizer of Eq.~\ref{eq:loss}: missing context cells enter as inputs, never as targets, so with an incomplete context $D$ stands for its observed cells and their missingness pattern (See Appendix~\ref{app:ctxmiss}). Optimizer and other constants are in Appendix~\ref{app:pretraining}.

\subsection{Inference: one procedure for every task}
\label{sec:method-inference}

\begin{wrapfigure}{r}{0.52\linewidth}
\vspace{-1em}
\small
\hrule height 0.9pt
\vspace{2pt}
\captionof{algorithm}{Unified conditional-query answering}
\vspace{-8.5pt}
\label{alg:inference}
\hrule\vspace{2pt}
\begin{algorithmic}[1]
\Require $c=(x_{\mathrm{obs}},D)$, $S$, grid $t_N>\dots>t_0=0$, $M$
\For{$m = 1,\dots,M$} 
  \State $z^{(m)}_{t_N,j} \gets \begin{cases} \textsc{mask}, & j \in S_{\mathrm{cat}} \\
  \sigma(t_N)\,\varepsilon_j, & j \in S_{\mathrm{num}} \end{cases}$
  \For{$i = N,\dots,1$}
    \State $z^{(m)}_{t_{i-1}} \sim p_\theta(\cdot \mid z^{(m)}_{t_i}, c)$ \Comment{Eq.~\ref{eq:step}}
  \EndFor
\EndFor
\State \Return $\{z^{(m)}_{t_0}\}_{m=1}^{M}$
\end{algorithmic}
\vspace{2pt}
\hrule height 0.9pt
\vspace{-1em}
\end{wrapfigure}

Because every generation task we consider is a conditional query against the same distribution, no task needs a procedure of its own: Algorithm~\ref{alg:inference} answers all of them with the same pretrained network, and a task is an instantiation of which cells form the target set $S$, where the reverse chain starts ($t_{\mathrm{start}}$), how many steps it takes ($N$), and how many trajectories it draws ($M$). Generation observes nothing and targets the whole row; imputation observes part of the row and target the rest; class-conditional generation observes the label and targets the rest. Per-task defaults are in Appendix~\ref{app:inference}.

Two properties make every instantiation valid and cheap (See Appendix~\ref{app:inference}).
i) \emph{The chain may start at any $t_\mathrm{start} \le 1$}: since $S$ and $t$ are independent in pretraining, every target set was seen at every noise level, and \textsc{mask} lies in the support of
every categorical marginal. 
ii) \emph{The context is encoded once}: query rows attend to the context but not the reverse, and the $t$ enters only the target cells, so the context keys and values are reused across steps, query rows, and trajectories.

\section{Evaluation}
\label{sec:experiments}
\vspace{-0.2em}

We evaluate \ourmethod{} through the following research questions:
\textbf{(RQ1. Generative performance)}
How does a single pretrained model without per-table training or tuning 
compare with models trained separately for each table on the three tasks?
\textbf{(RQ2. Posterior predictive fidelity)} 
How closely does the learned joint distribution 
match the posterior predictive distribution in Eq.~\ref{eq:ppd}, 
and how does this compare with the product of exact per-cell posterior predictives?
\textbf{(RQ3. Target-set performance)}
How does \ourmethod{} compare with per-dataset and in-context baselines 
as the target set varies in missing rate, mechanism, and pattern?
\textbf{(RQ4. Staged pretraining)}
How does the second pretraining stage 
affect performance on long and incomplete tables?

\begin{figure}[t]
    \centering
    \includegraphics[width=0.9\linewidth]{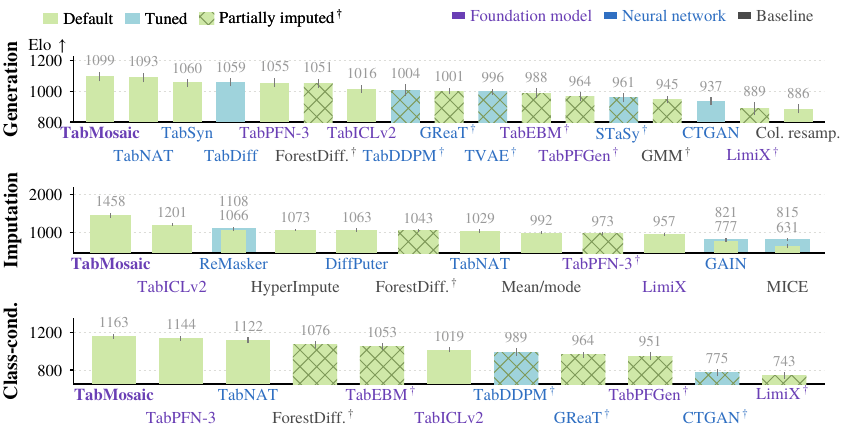}
    \caption{Elo with $95\%$ bootstrap intervals for generation and imputation on the $48$ tables and class-conditional generation on the $36$ classification tables. Bar color marks default and tuned entries, name color the model family, and hatching ($\dagger$) a method that did not complete every table, whose missing tables are scored as described in Appendix~\ref{app:protocol}. Full results are in Appendix~\ref{app:results}.
    }
    \label{fig:main}
\vspace{-1em}
\end{figure} 

\vspace{-0.2em}

\subsection{Experimental setup}
\label{sec:exp-setup}

\topic{Datasets.}
All real-data experiments use 48 tables from TabArena~\citep{erickson2026tabarena} with the
official splits; every stochastic method runs with three seeds, and methods are ranked by the official Elo procedure (Appendix~\ref{app:protocol}).

\topic{Metrics.} \emph{Generation} synthesizes as many rows as the training split and is scored by utility (train on synthetic, test on real, relative to training on real), $\alpha$-precision, $\beta$-recall, and authenticity~\citep{alaa2022faithful}. \emph{Class-conditional generation} conditions on the training split's own label column and is scored by utility, $\alpha$-precision and $\beta$-recall averaged over classes, and label consistency, i.e., the fraction of rows that a classifier trained on real data assigns to the requested class. \emph{Imputation} hides $30\%$ of the feature cells that are not already missing, completely at random, and is scored by min--max-scaled RMSE and error rate on the hidden cells and by the utility of the imputed table.

\topic{Baselines.}
Per-table models are tuned over ten configurations where their paper tunes
and run with their published configuration otherwise. For generation we compare with CTGAN,
TVAE, TabDDPM, STaSy, TabSyn, TabDiff, TabNAT, ForestDiffusion, and GReaT; for imputation
with a mean/mode imputer, MICE, GAIN, HyperImpute, ReMasker, DiffPuter, TabNAT, and
ForestDiffusion; and for class-conditional generation with TabDDPM, CTGAN, GReaT,
ForestDiffusion, TabNAT, and the training-free TabEBM and TabPFGen, each conditioning by its own
mechanism (See Table~\ref{tab:cond}). \ourmethod{} is one pretrained model used as is for every task and table.
In every task we also enter the in-context predictors TabPFN-3, TabICLv2, and LimiX with their official generation and imputation procedures (Appendix~\ref{app:posthoc}).

\subsection{RQ1. Generative performance}
\label{sec:exp-main}

\begin{figure}[t]
\centering
\includegraphics[width=0.9\linewidth]{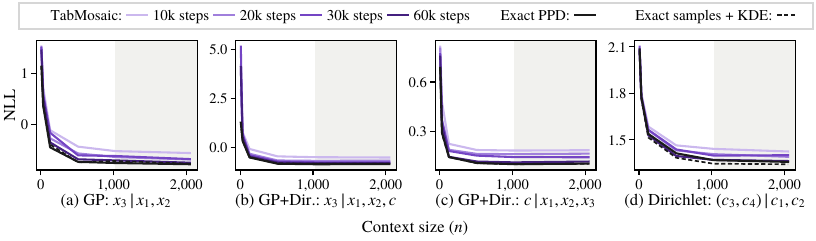}
\vspace{-1em}
\caption{NLL of held-out cells against the context size $n$, for a small model pretrained with
Eq.~\ref{eq:loss} on priors with a closed-form posterior predictive ($x$: numerical, $c$:
categorical cell; Appendix~\ref{app:ppd}). Solid black: exact posterior predictive; dashed:
$M$ exact samples under the same KDE as the model, i.e., the floor for a sampler. Shaded:
beyond the context sizes seen in pretraining.}
\label{fig:ppd}
\vspace{-1.5em}
\end{figure}

\begin{wrapfigure}{r}{0.43\linewidth}
    \centering
    \includegraphics[width=0.98\linewidth]{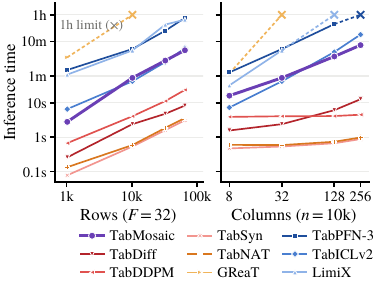}
    \vspace{-1em}
    \caption{Time to generate one synthetic table against the number of rows and of columns (median of three runs). $\times$: over the one-hour limit. Reds: per-table diffusion; oranges: per-table autoregressive; blues: in-context predictors sampled column by column.}
    \label{fig:latency}
    \vspace{-2em}
\end{wrapfigure}

\paragraph{Generation.}
\ourmethod{} ranks first by Elo ($1099$), with the best mean rank in utility ($4.1$, against $4.6$ for TabNAT), i.e., the rows it generates train the best predictors, whereas the per-table models, which are trained on the table, fit its support more tightly.
\vspace{-9pt}

\paragraph{Imputation.}
The margin is widest here: an Elo of $1458$, against $1201$ for TabICLv2 and $1108$ for tuned
ReMasker, and a mean error of $0.129$, against $0.146$ and $0.157$. The gap is largest on
categorical cells ($0.314$, against $0.375$ for TabICLv2), where every per-table imputer except TabNAT and ForestDiffusion stays above the mode imputer (0.443)
\vspace{-9pt}

\paragraph{Class-conditional generation.}
\label{sec:exp-cond}
The same model answers this task by observing the label cell. It ranks first by Elo ($1163$),
ahead of TabPFN-3 ($1144$) and TabNAT ($1122$), and its rows are about as separable by class as
real rows (label consistency $0.983$ of the real one), whereas TabPFGen over-separates them
($1.093$).
\vspace{-9pt}

\paragraph{Inference time.} Fig.~\ref{fig:latency} reports the time to generate one synthetic table as the rows and the columns grow. \ourmethod{} takes $53$~s at $10{,}240$ rows and $32$ columns, against $6$--$7$~min for TabPFN-3 and LimiX. The gap widens with the columns, since the in-context predictors sample one column at a time and re-encode the context for each, whereas \ourmethod{} denoises all target cells jointly. At $256$ columns, \ourmethod{} takes $8$~min, TabICLv2 $16$~min, and TabPFN-3 and LimiX exceed one hour. The per-table models sample in seconds once trained, but pay the training cost on every table. GReaT exceeds the limit from $10{,}240$ rows.

\subsection{RQ2. Posterior predictive fidelity}
\label{sec:exp-ppd}

\begin{wrapfigure}{r}{0.4\linewidth}
    \vspace{-1em}
    \centering
    \includegraphics[width=0.8\linewidth]{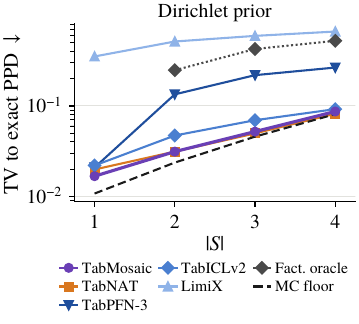}
    \vspace{-5pt}
    \caption{Distance to the exact posterior predictive as the target set $S$ grows on the Dirichlet prior. (Appendix~\ref{app:ppd}).}
    \label{fig:ppd-real}
    \vspace{-1.5em}
\end{wrapfigure}

Corollary~\ref{cor:ppd} states what the objective targets. Following \citet{muller2022transformers}, we test it on priors whose posterior predictive is available in closed form: a Gaussian Process (GP), a GP with a Dirichlet-distributed categorical column, and a Dirichlet prior over the joint of four
categorical columns (See Appendix~\ref{app:ppd}). Every panel carries the Monte-Carlo floor, i.e., the statistic between two sets of $M$ exact samples, and Fig.~\ref{fig:ppd-real} adds a factorized oracle, i.e., the product of the exact per-cell posteriors, which bounds any method that fills the cells independently.

\topic{The objective reaches the posterior predictive.}
We pretrain a small model with our objective Eq.~\ref{eq:loss} on a mixture of these priors and evaluate it. As shown in the Fig.~\ref{fig:ppd}, the NLL approaches that of the exact posterior predictive on every prior, e.g., on the GP at $n=1,024$, it moves from $-0.53$ to $-0.71$ against $-0.77$ for the exact one, and on the categorical cells of the GP+Dirichlet prior it reaches $0.114$ against $0.103$. The same holds beyond the context sizes seen in training ($n=2048$).

\topic{The joint over a target set.}
We then apply the pretrained \ourmethod{}, which has never seen this prior, to hidden sets $|S| \in \{1, \dots, 4\}$, i.e., from imputation to generation (Fig.~\ref{fig:ppd-real}). The factorized oracle degrades as $|S|$ grows, i.e., the dependence among the hidden cells matters. \ourmethod{} stays within $1.0$--$1.3$ times the floor for $|S| \ge 2$ without training, on par with TabNAT, which is trained on each context table. TabICLv2 is within $1.1$--$2.0$, and TabPFN-3 and LimiX are farther ($3.3$--$5.7$ and $8$--$22$, against $6.5$--$10.5$ for the oracle). \ourmethod{} obtains this joint in one pass without an ordering, whereas the in-context predictors assemble it column by column (Fig.~\ref{fig:latency}).
\vspace{-0.4em}

\subsection{RQ3. Target-set performance}
\label{sec:exp-targets}

\begin{figure}[t]
\centering
\begin{minipage}[t]{0.46\linewidth}
    \vspace{7pt}
    \footnotesize
    \centering
    \captionof{table}{Error on hidden cells relative to a mean/mode imputer (below $1$ means better than mean/mode), for the four best of nine methods on the $24$ tables that all nine methods completed. All methods in Table~\ref{tab:targets-full}.}
    \label{tab:targets}
    \setlength{\tabcolsep}{1pt}
    \begin{tabular}{@{}llcccc@{}}
    \toprule
     & & TabMosaic & TabICLv2 & TabNAT & DiffPuter \\
    \midrule
    \multirow{4}{*}{\rotatebox{90}{\scriptsize MCAR rate}} & $10\%$ & \textbf{0.622} & 0.729 & 0.777 & 0.901 \\
     & $30\%$ & \textbf{0.675} & 0.812 & 0.860 & 0.966 \\
     & $50\%$ & \textbf{0.770} & 0.921 & 0.932 & 0.955 \\
     & $70\%$ & \textbf{0.843} & 0.974 & 0.993 & 1.005 \\
    \midrule
    \multirow{2}{*}{\rotatebox{90}{\scriptsize $30\%$}} & MAR & \textbf{0.628} & 0.807 & 0.887 & 0.887 \\
     & MNAR & \textbf{0.704} & 0.824 & 0.871 & 0.866 \\
    \midrule
    \multirow{3}{*}{\rotatebox{90}{\scriptsize columns}} & $|S|{=}1$ & \textbf{0.523} & 0.575 & 0.780 & 0.826 \\
     & $|S|{=}2$ & \textbf{0.600} & 0.776 & 0.895 & 0.822 \\
     & $|S|{=}4$ & \textbf{0.767} & 0.888 & 0.960 & 0.912 \\
    \bottomrule
    \end{tabular}
\end{minipage}\hfill
\begin{minipage}[t]{0.52\linewidth}
    \vspace{0pt}
    \centering
    \includegraphics[width=0.85\linewidth]{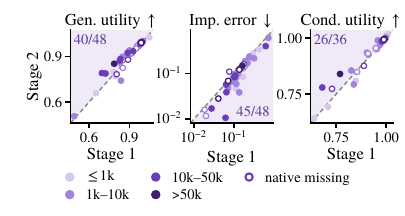}\\[-7pt]
    \includegraphics[width=0.85\linewidth]{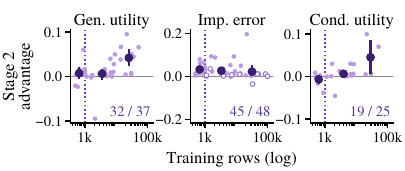}
    \vspace{-0.5em}
    \caption{Stage 1 vs.\ Stage 2: per-table scores (top) and the Stage 2 advantage against table size (bottom). Native missing cells are pre-filled for Stage 1.}
\label{fig:stage}
\end{minipage}
\vspace{-1.5em}
\end{figure}

Algorithm~\ref{alg:inference} takes any target set without adaptation.
We vary the target set along three axes on real tables, with the training split as context and
the test rows as queries (the out-of-sample setting of
\citealt{zhang2025tabnat}): the \emph{missing rate} $\{10, 30, 50, 70\%\}$, the \emph{mechanism} that selects the hidden cells
(MCAR, MAR, MNAR); and the \emph{pattern}, random cells against $|S| \in \{1, 2, 4\}$ whole columns
hidden in every query row. 
Table~\ref{tab:targets} reports the error on the hidden cells relative to a mean/mode imputer. \ourmethod{} has the lowest error in every setting. Its error grows with the missing rate ($0.62$ to $0.84$), where the baselines fall back to the imputer ($\approx 1$), and is flat across the mechanisms. 
\vspace{-0.4em}

\subsection{RQ4. Staged pretraining}
We compare the model after Stage~1 with the final one on the same tables and tasks to isolate what Stage~2 adds. Fig.~\ref{fig:stage} shows that Stage~2 lowers the imputation error on $45$ of the $48$ tables, from a mean of $0.156$ to $0.129$, including $10$ of the $11$ tables whose context has missing values, which Stage~1 reads only after a mean/mode pre-fill. Generation and class-conditional utility improve on $40$ of $48$ and $26$ of $36$ tables, with the gain concentrated on tables above $10$k rows, i.e., the lengths that Stage~1 never saw.
\vspace{-0.4em}

\section{Conclusion}
\label{sec:conclusion}
\vspace{-0.4em}
We posed tabular generation, imputation, and class-conditional generation as one problem, learning the joint conditional distribution over an arbitrary set of target cells on an unseen table, and showed that a diffusion objective over the target set is the natural fit: it defines a valid joint for every mixed-type target set without an ordering, upper-bounds the conditional NLL, and, under the pretraining prior, is minimized by the posterior predictive distribution. \ourmethod{}, a single model pretrained with this objective in two stages, ranks first on all three tasks across $48$ datasets against $20$ baselines, approaches the exact posterior predictive on synthetic priors, and serves arbitrary target sets as is.


\clearpage
\subsection*{AI use statement}
We used generative AI tools to discuss the experimental design, to write and run the code for experiments and figures, to check the proofs in Appendix~\ref{app:proofs}, and to draft and edit parts of the text. The research questions, the method, and all decisions originated with the authors, who reviewed every AI-assisted output and take full responsibility for the content.
{
    \bibliography{iclr2027_conference}
    \bibliographystyle{iclr2027_conference}
}

\clearpage
\appendix
\addtocontents{toc}{\protect\setcounter{tocdepth}{2}}
\renewcommand{\contentsname}{Table of Contents}
\section*{Appendix}
\tableofcontents
\clearpage



\section{Proofs}
\label{app:proofs}

Throughout the appendix, we write $c = (x_{\mathrm{obs}}, D)$ for the conditioning information,
$S = S_{\mathrm{num}} \cup S_{\mathrm{cat}}$ for the target set, and
$\mathcal{Z} = \mathbb{R}^{|S_{\mathrm{num}}|} \times \prod_{j \in S_{\mathrm{cat}}}
\mathcal{V}_j$ for the mixed state space, where $\mathcal{V}_j$ is the class set of column $j$
augmented with the \textsc{mask} token.

\subsection{Proof of Proposition~\ref{prop:valid} (validity)}
\label{app:proof-valid}

\begin{proof}
Each factor of the reverse step in Eq.~\ref{eq:step} is a normalized Gaussian density or a
normalized categorical PMF in its own cell variable, for \emph{any} value of the conditioning
arguments $(z_t, c, t)$ and any parameter $\theta$; hence Eq.~\ref{eq:step} is a product
probability measure on $\mathcal{Z}$ with respect to the product of the Lebesgue and counting
measures, i.e., a Markov kernel. The terminal law $p(z_{t_N})$ is a fixed normalized product law,
i.e., $\mathcal{N}(0, \sigma_{\max}^2)$ per numerical cell and the \textsc{mask} point mass per
categorical cell. A finite composition of Markov kernels applied to a normalized initial law
defines a normalized joint law of $(z_{t_N}, \dots, z_{t_0})$, and marginalizing the intermediate
states leaves a normalized law of $z_{t_0} = x_S$, which is Eq.~\ref{eq:model}.
\end{proof}

We emphasize the two distinct roles of the product structure. The \emph{loss} decomposes
additively over cells because the two \emph{step} distributions, Eq.~\ref{eq:forward}
and Eq.~\ref{eq:step}, factorize given $z_t$; the \emph{model}, Eq.~\ref{eq:model}, does not
factorize across cells or modalities at the endpoint, because the factors of consecutive steps
are coupled through the full mixed state.

\subsection{Proof of Theorem~\ref{thm:elbo} (conditional ELBO)}
\label{app:proof-elbo}

\begin{proof}
Fix $(x, S, c)$ and the time grid $1 = t_N > \dots > t_0 = 0$. Writing $q$ for the forward law of
$(z_{t_1}, \dots, z_{t_N})$ given $x_S$ and applying Jensen's inequality to
Eq.~\ref{eq:model} --- the only inequality in the proof ---
\begin{equation}
-\log p_\theta(x_S \mid c)
\;\le\; \mathbb{E}_q\!\left[ -\log
\frac{p(z_{t_N}) \prod_{i=N}^{1} p_\theta(z_{t_{i-1}} \mid z_{t_i}, c)}
{q(z_{t_1}, \dots, z_{t_N} \mid x_S)} \right],
\end{equation}
which rearranges into the standard decomposition
\begin{equation}
\sum_{i=2}^{N} \mathbb{E}\,
\mathrm{KL}\big( q(z_{t_{i-1}} \mid z_{t_i}, x_S) \,\big\|\, p_\theta(z_{t_{i-1}} \mid z_{t_i}, c) \big)
\;+\; \mathbb{E}\,\mathrm{KL}\big( q(z_{t_N} \mid x_S) \,\big\|\, p(z_{t_N}) \big)
\;+\; \mathbb{E}\big[ -\log p_\theta(x_S \mid z_{t_1}, c) \big].
\end{equation}
The prior term is $\theta$-free and vanishes as $\sigma_{\max}$ grows (numerical) and as the
terminal mixing completes (categorical); together with the discretization residue of the
reconstruction term, it forms the constant $C$.

For each step term, \emph{both} arguments of the KL factorize over cells: the forward posterior
by the cell-wise independence of Eq.~\ref{eq:forward}, and the model by the definition of the
step, Eq.~\ref{eq:step}. The KL divergence between two product measures is the sum of the
per-factor KL divergences, so every step term is a sum of per-cell terms, which we group by
modality.

For a numerical cell $j$, both factors are Gaussians with the same variance, differing only in
the location parameter, which is a linear function of the clean value; the per-step KL is
therefore a quadratic $\kappa_i\, (\hat{x}_{\theta,j}(z_{t_i}, c; t_i) - x_j)^2$ with a
$\theta$-free coefficient $\kappa_i$ determined by the schedule. Summing over the steps and
passing to the continuous-time limit yields
$\mathbb{E}_t\, w_{\mathrm{lik}}(t)\, (\hat{x}_{\theta,j} - x_j)^2$ with the likelihood weight
$w_{\mathrm{lik}}(t) = \sigma'(t) / \sigma(t)^3$; with this weight, the numerical part of the
bound is exact up to discretization. Replacing $w_{\mathrm{lik}}$ by the EDM weight
$\lambda(\sigma(t))$ (Appendix~\ref{app:kernels}) multiplies each $(t, z_t)$-indexed term by a
positive factor. Since the minimization over $\hat{x}_{\theta,j}(z_t, c; t)$ is pointwise in
$(t, z_t)$, every positive weighting has the same minimizer
$\hat{x}_{\theta,j} = \mathbb{E}[x_j \mid z_t, c]$; the reweighting changes the optimization
geometry, not the target.

For a categorical cell $j$, the per-step KL divergences under the GIDD kernel telescope, in the
continuous-time limit, into the exact weighted form
$\mathbb{E}_t\, w(z_{t,j})\, [D_{\mathrm{KL},j} + D_{\mathrm{IS},j}]$ of
\citet{vonrutte2025generalized}, where the weights and both divergence terms are available in
closed form (Appendix~\ref{app:kernels}). Collecting the numerical cells into
$\mathcal{L}_{\mathrm{num}}$ and the categorical cells into $\mathcal{L}_{\mathrm{cat}}$ gives
the statement.
\end{proof}

\subsection{Proof of Corollary~\ref{cor:ppd} (the objective targets the posterior predictive distribution)}
\label{app:proof-ppd}

\begin{proof}
Let tables be drawn from the prior, $(x \cup D) \sim p(\mathcal{D})$, with the target pattern
drawn independently of the data, $S \perp (x, D)$. Taking expectations of the pointwise bound of
Theorem~\ref{thm:elbo},
\begin{equation}
\mathbb{E}\big[ \ell^\star_\theta \big] + C
\;\ge\; \mathbb{E}\big[ -\log p_\theta(x_S \mid x_{\mathrm{obs}}, D) \big].
\end{equation}
Conditionally on $(S, x_{\mathrm{obs}}, D)$, the clean target $x_S$ is distributed as the
posterior predictive $p(x_S \mid x_{\mathrm{obs}}, D)$ of Eq.~\ref{eq:ppd}, because the query row
and the context rows are i.i.d.\ given the mechanism $\phi$. The cross-entropy decomposition
then gives
\begin{equation}
\mathbb{E}\big[ -\log p_\theta(x_S \mid x_{\mathrm{obs}}, D) \big]
= \mathbb{E}\big[ H\big( p(\cdot \mid x_{\mathrm{obs}}, D) \big) \big]
+ \mathbb{E}\big[ \mathrm{KL}\big( p(\cdot \mid x_{\mathrm{obs}}, D) \,\big\|\,
p_\theta(\cdot \mid x_{\mathrm{obs}}, D) \big) \big],
\end{equation}
where the entropy term is independent of $\theta$; the constant $C'$ of the statement collects
$C$ and this entropy. Minimizing the loss therefore minimizes an upper bound on the expected KL
divergence to the posterior predictive. That $\mathcal{L}$ and $\mathbb{E}[\ell^\star_\theta]$
share their minimizer follows from the proof of Theorem~\ref{thm:elbo}: the two differ only in
the positive per-$(t, z_t)$ weights of the numerical terms. When $S$ is a single cell and the
model class is an explicit classifier, the diffusion ELBO is replaced by the exact
log-likelihood and the bound collapses to the cross-entropy identity of
\S\ref{sec:related-icl}; the corollary is the same statement with the single cell generalized to
an arbitrary mixed-type cell set and the equality relaxed by the ELBO.
\end{proof}

\subsection{Proof of Corollary~\ref{cor:min} (The posterior predictive minimizes the objective)}
\label{app:proof-min}

\begin{proof}[Proof.]
Under the assumptions, $(x_S, z_t)$ in Eq.~\ref{eq:loss} is distributed as $p(x_S \mid c)\,q_t(z_t \mid x_S)$. Conditioned on $(z_t, c)$, the numerical term is a squared loss, minimized by the posterior mean, and the categorical term is the GIDD ELBO term, which is minimized by the true posterior PMF~\citep{vonrutte2025generalized}. Both minimizers are pointwise in $(z_t, c, t)$ and independent of the positive weights, hence they minimize $\mathcal{L}$ and $\mathbb{E}[\ell^\star_\theta]$ alike. With this denoiser every factor of Eq.~\ref{eq:step} is the true reverse transition of the forward process started from $p(x_S \mid c)$, and as $N \to \infty$ the per-step KL terms of the ELBO vanish~\citep{song2021score, vonrutte2025generalized}, so the composition has marginal $p(x_S \mid c)$ at $t = 0$.
\end{proof}

\section{Kernel details: EDM and GIDD}
\label{app:kernels}

\paragraph{Shared time.}
One time $t \sim \mathcal{U}[\epsilon, 1 - \epsilon]$ with $\epsilon = 10^{-3}$ is drawn per table and shared by the numerical and the categorical cells of its target set. The numerical noise level follows the EDM schedule $\sigma(t) = \big(\sigma_{\min}^{1/\rho} + t\,(\sigma_{\max}^{1/\rho} - \sigma_{\min}^{1/\rho})\big)^{\rho}$ with $\sigma_{\min} = 0.002$, $\sigma_{\max} = 80$, and $\rho = 7$~\citep{karras2022elucidating}; the categorical mixing coefficients are those of the GIDD schedule below at the same $t$.

\paragraph{Numerical cells (EDM).}
Numerical cells follow \citet{karras2022elucidating}: the forward kernel is $q_\sigma(z \mid x) = \mathcal{N}(z; x, \sigma^2)$ and the denoiser uses the EDM preconditioning with data scale $\sigma_d = 1$, since every numerical column is mapped to a standard normal marginal by a per-column quantile transform fit on the observed context cells (Appendix~\ref{app:pretraining}). The training weight is the EDM weight $\lambda(\sigma) = (\sigma^2 + \sigma_d^2) / (\sigma \sigma_d)^2$, which equalizes the gradient magnitude across noise levels and, as shown in Appendix~\ref{app:proof-elbo}, shares its minimizer $\mathbb{E}[x \mid z, c]$ with the likelihood weight at every $(z, \sigma)$.

\paragraph{Categorical cells (GIDD).}
\label{app:kernels-gidd}
Categorical cells follow the interpolating family of \citet{vonrutte2025generalized}, whose forward marginal is Eq.~\ref{eq:forward} with a mixing distribution $\pi_t$ that blends the \textsc{mask} point mass with uniform noise; we use their mixing rate $p_u = 0.1$ and $\gamma = 1$, under which the uniform component vanishes at both ends, so that $\alpha_0 = 1$ (clean) and the terminal law is exactly the \textsc{mask} point mass. The per-cell loss carries the exact continuous-time ELBO weights of their Theorem~3.7, $\mathbb{E}_t\, w(z_{t,j})\, [D_{\mathrm{KL},j} + D_{\mathrm{IS},j}]$, where $D_{\mathrm{KL},j}$ compares the true and the model one-step posteriors, $D_{\mathrm{IS},j}$ is the importance-sampled correction for non-\textsc{mask} intermediate states, and the weight $w$ is available in closed form. Three properties matter for our construction: i) the weights are exact, so the categorical block enters Theorem~\ref{thm:elbo} without any reweighting caveat; ii) unlike masked (absorbing) diffusion~\citep{austin2021structured, sahoo2024simple}, where an unmasked token can never change again, the uniform component lets the reverse process revise an already-decoded token; and iii) \textsc{mask} is a legal state at every $t$ and carries no information about the clean value, which is what lets a categorical target cell enter the reverse process at any $t$ (Appendix~\ref{app:inference}).

\section{Multi-corruption estimator of the objective}
\label{app:ensemble}

Theorem~\ref{thm:elbo} and Corollary~\ref{cor:ppd} concern the population objective
\begin{equation}
\mathcal{J}(\theta) = \mathbb{E}_{(x \cup D) \sim p(\mathcal{D})}\;
\mathbb{E}_{S}\; \mathbb{E}_{t}\;
\mathbb{E}_{z_t \sim q_t(\cdot \mid x_S)}
\big[ \ell_\theta(z_t, t; x_S, c) \big],
\qquad
\ell_\theta = \tfrac{1}{|S|}\big[\mathcal{L}_{\mathrm{num}} + \mathcal{L}_{\mathrm{cat}}\big],
\label{eq:population}
\end{equation}
which is $\mathcal{L}$ of Eq.~\ref{eq:loss}. Training replaces the inner expectations by
samples. Per optimization step, we draw $G$ tables; for table $g$, we draw \emph{one} time $t_g$
and \emph{one} target pattern $S_g$, and then $K$ independent corruption realizations of the
same $(t_g, S_g)$:
\begin{equation}
z^{(g,k)} \overset{\mathrm{iid}}{\sim} q_{t_g}\big(\cdot \mid x_{S_g}\big),
\qquad
\widehat{\mathcal{J}}_{G,K}(\theta)
= \frac{1}{G} \sum_{g=1}^{G} \frac{1}{K} \sum_{k=1}^{K}
\ell_\theta\big(z^{(g,k)}, t_g;\; x_{S_g}, c_g\big),
\label{eq:estimator}
\end{equation}
whose inner average is $\hat{\ell}_K$ of Eq.~\ref{eq:multi}. $K = 1$ is the standard
single-corruption estimator; $K > 1$ averages the \emph{losses} of the $K$ copies, not their
outputs.

\begin{proposition}[Unbiasedness and variance]
\label{prop:estimator}
For every $K \ge 1$, $\mathbb{E}[\widehat{\mathcal{J}}_{G,K}] = \mathcal{J}$ and
$\mathbb{E}[\nabla_\theta \widehat{\mathcal{J}}_{G,K}] = \nabla_\theta \mathcal{J}$, and
\begin{equation}
\mathrm{Var}\big(\nabla_\theta \widehat{\mathcal{J}}_{G,K}\big)
= \frac{1}{G}\left( V_{\mathrm{out}} + \frac{V_{\mathrm{in}}}{K} \right),
\qquad
\begin{aligned}
V_{\mathrm{out}} &= \mathrm{Var}_{(x, D, S, t)}\big( \mathbb{E}_z[\nabla_\theta \ell_\theta] \big),\\
V_{\mathrm{in}} &= \mathbb{E}_{(x, D, S, t)}\big[ \mathrm{Var}_z(\nabla_\theta \ell_\theta) \big].
\end{aligned}
\end{equation}
Hence $K$ reduces exactly the corruption-conditional component $V_{\mathrm{in}}$ by $1/K$, while
$G$ reduces both, and the minimizer is unchanged by $K$.
\end{proposition}

\begin{proof}
Each summand of Eq.~\ref{eq:estimator} is an unbiased sample of Eq.~\ref{eq:population} by
construction, which gives unbiasedness, and, under standard regularity, the same for the
gradient. The variance follows from the law of total variance conditioned on $(x, D, S, t)$: the
inner mean of $K$ i.i.d.\ draws has conditional variance
$\mathrm{Var}_z(\nabla_\theta \ell_\theta) / K$, and the $G$ tables are independent.
\end{proof}

\section{Missingness in the in-context set}
\label{app:ctxmiss}

The second pretraining stage additionally corrupts the \emph{conditioning} side: some cells of
the context rows are hidden from the model. This changes the conditioning object, not the bound.

\paragraph{Masking distribution.}
Let $R \in \{0,1\}^{|D| \times F}$ be an observation mask over context cells. For each table we
draw a rate $\rho \sim \mathcal{U}(0, \rho_{\max}(s))$, where $s$ is the training step and
$\rho_{\max}(s) = 0.2 \min(1, s / 2{,}500)$; a pattern, i.e., i.i.d.\ cells (probability $0.5$),
a random subset of columns (0.3) or of rows (0.2) hidden at a higher per-cell rate
$q \sim \mathcal{U}(0.5, 1)$ and included with probability $\rho / q$, so that the expected
fraction of hidden cells is $\rho$ under every pattern; and a mechanism, i.e., MCAR
(probability $0.5$), MAR (0.25) or MNAR (0.25). Under MAR, the per-cell rate of a row is tilted
by the within-column rank of one randomly chosen driver column, which stays fully observed;
under MNAR, the rate of a cell is tilted by the rank of its own value (self-masking). The tilt
amplitude is drawn uniformly in $[0, 2\min(\rho, 1 - \rho)]$, so the expected rate remains
$\rho$ and every cell probability stays in $[0, 1]$. The label column is masked like any other
column. The masked context and the conditioning object are
\begin{equation}
\widetilde{D}_{ij} =
\begin{cases}
D_{ij}, & R_{ij} = 1,\\
\emptyset, & R_{ij} = 0,
\end{cases}
\qquad
\widetilde{c} = \big(x_{\mathrm{obs}}, \widetilde{D}, R\big),
\label{eq:maskdist}
\end{equation}
where $\emptyset$ is realized in the network as the \textsc{mask} token (categorical) or the
value $0$ (numerical) \emph{together with} the third observation state in
$\{\text{target}, \text{observed}, \text{missing-in-context}\}$, so that the network reads both
the observed values and the pattern $R$. The per-column quantile transform of the numerical
columns is fit on the observed context cells only, so $\widetilde{c}$ is a function of the
observed part of $D$ and of $R$. Missing cells of the \emph{query} row need no new state: at
inference, they are precisely the cells to be imputed, i.e., they belong to $S$. The training
objective is Eq.~\ref{eq:population} with $c$ replaced by $\widetilde{c}$; we denote it
$\mathcal{J}_s(\theta)$.

\section{Pretraining details}
\label{app:pretraining}

\paragraph{Synthetic prior.}
Tables are drawn from the structural-causal-model prior of TabICL~v2~\citep{qu2026tabiclv2}
with the two changes of \S\ref{sec:method-pretraining}: the label is an ordinary column whose
cardinality follows the same law as the categorical features (up to $200$ classes), and half of
the tables have a numerical label. We use the corrected sampling of categorical cardinalities of
its implementation, graphs with $2$ to $32$ nodes, and its filters that discard graphs and
tables whose label is unpredictable from the features. Every numerical column is mapped to a
standard normal marginal by a per-column quantile transform fit on the observed context cells;
values are mapped back by the inverse transform at inference.

\paragraph{Target pattern.}
For each table, with probability $0.4$ all cells of the query rows are targets (generation);
otherwise a rate $\rho \sim \mathcal{U}(0, 1)$ is drawn and each query cell is a target
independently with probability $\rho$. The pattern is drawn independently of the cell values
and of $t$ (\S\ref{sec:method-objective}).

\paragraph{Table shapes.}
Stage~1 uses $T = 1{,}024$ rows per table, $F \sim \log\mathcal{U}(1, 100)$ features, and a
context fraction drawn from $\mathcal{U}(0.3, 0.9)$; tables have no missing values. Stage~2 draws
the number of rows first, $T \sim \log\mathcal{U}(400, 10{,}240)$, and then the number of
features, $F \sim \log\mathcal{U}(4, \min(200, \lfloor B / T \rfloor))$, with the cell budget
$B = 250{,}000$, so that $T \times F \le B$; the context fraction is $0.8$. Tables in a
micro-batch of four share one shape.

\begin{table}[ht]
\caption{Pretraining configuration. Both stages use $8$ GPUs.}
\label{tab:app-stages}
\begin{center}
\begin{footnotesize}
\begin{tabular}{@{}lll@{}}
\toprule
 & Stage 1 & Stage 2 \\
\midrule
Rows $T$ & $1{,}024$ & $\log\mathcal{U}(400, 10{,}240)$ \\
Features $F$ & $\log\mathcal{U}(1, 100)$ & $\log\mathcal{U}(4, \min(200, B/T))$, $B = 250{,}000$ \\
Context fraction & $\mathcal{U}(0.3, 0.9)$ & $0.8$ \\
Context missingness & --- & $\rho \sim \mathcal{U}(0, \rho_{\max})$, $\rho_{\max}$: $0 \to 0.2$ over $2{,}500$ steps \\
Corruptions $K$ per table & $1$ & $4$ \\
Steps / tables seen & $125$k / $32$M & $10$k / $2.56$M \\
Global batch (tables) & $256$ & $256$ \\
Peak learning rate & $8 \times 10^{-4}$ & $10^{-4}$ \\
Warm-up & $1\%$ & $2\%$ \\
Initialization & random & stage-1 final \\
\bottomrule
\end{tabular}
\end{footnotesize}
\end{center}
\end{table}

\paragraph{Stages.}
Table~\ref{tab:app-stages} summarizes the two-stage recipe. Stage~2 warm-starts from the final
stage-1 checkpoint and introduces the long tables, the context-missingness distribution of
Appendix~\ref{app:ctxmiss}, and the multi-corruption estimator of Appendix~\ref{app:ensemble}
with $K = 4$. The evaluated model is the stage-2 checkpoint at step $10{,}000$; no per-table
training, tuning, or selection is applied anywhere in the paper.

\paragraph{Architecture and optimization.}
The denoiser is the column--row attention transformer of \S\ref{sec:method-arch} with $12$
blocks, width $384$, $8$ heads, a feed-forward expansion of $2$, pre-normalization, and
normalized queries and keys ($29$M parameters). Each cell token combines a type-specific value
embedding (a linear map for numerical cells; a class embedding with a shared vocabulary of
$201$ entries, the last of which is \textsc{mask}, for categorical cells), the three-state
observation embedding, and, for target cells only, the embedding of $t$. Optimization uses Muon
for the two-dimensional weights and AdamW for the remaining parameters at $0.2\times$ the
learning rate, with $\beta_1 = 0.9$, weight decay $0.01$, a cosine schedule with linear warm-up
decaying to $10^{-7}$, gradient clipping at $10$, bf16 precision, and activation checkpointing.

\section{Inference details}
\label{app:inference}

All the tasks run Algorithm~\ref{alg:inference} with the same network; Table~\ref{tab:app-tasks}
lists the settings. The context is the training split of the table, capped at $8{,}192$ rows: a
random subset for generation and a label-stratified random subset for imputation. Cells of the
context that are missing in the table are presented in the missing-in-context state
(Appendix~\ref{app:ctxmiss}); the quantile transform is fit on the observed cells.

\begin{table}[ht]
\caption{Settings per task. $t_{\mathrm{start}}$ is the time at which the reverse process starts,
$N$ the number of steps, and $M$ the number of trajectories per query row; the context is
encoded once per query batch.}
\label{tab:app-tasks}
\begin{center}
\begin{footnotesize}
\setlength{\tabcolsep}{3pt}
\begin{tabular}{@{}llllll@{}}
\toprule
Task & $S$ & $x_{\mathrm{obs}}$ & $t_{\mathrm{start}}$ & Sampler (num.) / $N$ & $M$; output \\
\midrule
Generation & all query cells & --- & $1$ & deterministic / $128$ & $1$; one row per trajectory \\
Class-cond.\ gen. & feature cells & label cell & $1$ & deterministic / $128$ & $1$; one row per trajectory \\
Imputation & missing cells & observed cells & $0.5$ & ancestral / $16$ & $8$; mean (num.), vote (cat.) \\
\bottomrule
\end{tabular}
\end{footnotesize}
\end{center}
\end{table}

\paragraph{Starting the reverse process at $t_{\mathrm{start}} < 1$.}
Since $S \perp t$ during pretraining, the reverse process can be entered at any
$t_{\mathrm{start}} \le 1$. Categorical target cells enter at \textsc{mask}, which lies in the
support of the forward marginal at every $t$ and carries no information about the clean value
(Appendix~\ref{app:kernels-gidd}). Numerical target cells enter at
$z_{t_{\mathrm{start}}} = \hat{x}^{(0)} + \sigma(t_{\mathrm{start}})\, \varepsilon$, where
$\hat{x}^{(0)} = \hat{x}_\theta(\sigma_{\max}\varepsilon'; \sigma_{\max}, c)$ is the output of
one denoiser evaluation at the terminal noise level, i.e., the conditional mean. Imputation uses $t_{\mathrm{start}} = 0.5$ with $N = 16$ ancestral
steps, which we selected on a development checkpoint: per-cell error metrics saturate at a
handful of steps, whereas the distribution-level metrics of generation prefer the full
trajectory, so generation starts at $t_{\mathrm{start}} = 1$ with $N = 128$.

At $N{=}1$ a numerical target cell returns the conditional mean and a categorical
target cell is drawn from the denoiser PMF, each cell independently of the others; this is
property~iii) of \S\ref{sec:method-model}.
The
coupling between the cells of $S$ is created by the iteration: with $N > 1$ every step reads the
current state of every other target cell (Eq.~\ref{eq:step}).

\paragraph{Point estimates from samples.}
Imputation is scored by per-cell error metrics (Appendix~\ref{app:protocol}), which require one
value per cell. We take the mean over the $M = 8$ trajectories for a numerical cell and the
majority vote for a categorical cell (ties are broken by the mode of the observed cells of the
column, then the lowest class index). Distribution-level metrics use the trajectories as they
are, one sample per trajectory.

\paragraph{Generation.}
Generation and class-conditional generation run the full reverse process from the terminal law
over all the target cells jointly; for class-conditional generation the label cell of every
query row is observed and set to the requested class. The denoiser PMF of a categorical cell is
used at temperature $1$: lowering it trades diversity for marginal sharpness and degrades
distribution-level metrics. Each evaluation seed regenerates the entire table, i.e., seeds are
whole-table replicates.

\paragraph{Cost.}
The context is encoded once per query batch and reused across steps, query rows, and trajectories (\S\ref{sec:method-inference}); query rows are processed in blocks of $8{,}192$. The median wall-clock of generation is $25$~s per table on one GPU, whereas the per-table generators train on every table, e.g., TabNAT for a median of $1,696$~s per table over the $48$ tables and three seeds; the scaling with rows and columns is in Fig.~\ref{fig:latency}.

\section{Evaluation protocol}
\label{app:protocol}

\paragraph{Benchmark.}
All the evaluations use the TabArena suite~\citep{erickson2026tabarena} under its single-split (lite) protocol, i.e., the first official outer split of each dataset. Three of the 51 tables have more than 1,000 features; we exclude them from every task and every method, which leaves $48$ tables with $498$ to $100{,}000$ training rows and up to $212$ features. Class-conditional generation uses the $36$ classification tables among them. We report Elo, ranks, and means without statistical significance claims. In the per-size analyses, tables are grouped by training rows into $\le 1$k ($12$ tables), $1$k--$10$k ($21$), $10$k--$50$k ($11$), and $> 50$k ($4$).

\paragraph{Scoring predictors.}
Utility and label consistency are measured with three predictors --- LightGBM, CatBoost, and
TabPFN --- at their default configurations; the score of a table is the mean over the three. We
use three predictors so that no single inductive bias decides the ranking; the predictors are
never tuned on synthetic data.

\paragraph{Generation.}
Every generator is given the official training split as its training data (per-table models) or
as its context (\ourmethod{} and the in-context generators) and synthesizes $n_{\mathrm{train}}$ rows.
Scores are i) \emph{utility}: the official metric of the scoring predictors trained on synthetic
rows and tested on the real test split, divided by the same metric when trained on the real
training rows (the ratio is inverted for the metrics where lower is better, i.e., log loss and RMSE); ii)
$\alpha$-precision and $\beta$-recall~\citep{alaa2022faithful}; and iii)
authenticity~\citep{alaa2022faithful}. Column-level statistics --- Wasserstein-1 on
min-max--scaled numerical columns, total variation on categorical columns, and the mean
absolute difference of pairwise correlations --- are reported as diagnostics
(Appendix~\ref{app:results}). The pool comprises i) per-table generators
\citep{kotelnikov2023tabddpm, kim2023stasy, zhang2024tabsyn, shi2025tabdiff,
zhang2025tabnat, jolicoeur2024generating, borisov2022language}, tuned over $10$ configurations from
their published search spaces where their paper tunes and run with their published configuration
otherwise (TabNAT, ForestDiffusion, and GReaT publish no search space and are run
default-only); ii) two reference baselines, a Gaussian mixture over the row and independent
column resampling, which fix the scale of the Elo; and iii) in-context generators that chain a
pretrained predictor over the columns without training on the table
\citep{grinsztajn2026tabpfn, qu2026tabiclv2, zhang2025limix}, run with their official generation
procedures (Appendix~\ref{app:posthoc}). Every stochastic method is run with $3$ seeds; each seed
regenerates the whole table, scores are computed per seed, and the \emph{scores} are averaged. On classification tables, TabDDPM is trained with the label as its
conditioning input and samples labels from the empirical class proportions, and CTGAN samples
its discrete-column condition vector from the data frequencies; on regression tables both treat
the target as a numerical feature, as in their papers.

\paragraph{Imputation.}
Missing cells are injected into the feature cells of the training split completely at random at
a $30\%$ rate with $3$ mask seeds; the label column is never injected and is passed to every
imputer as an ordinary observed column; natural missingness is passed as missing, excluded from
injection, and never scored. Each imputer fits on the incomplete table and fills it, and only the
injected cells are scored: min-max--scaled RMSE on numerical cells and the misclassification rate
on categorical cells, averaged per table; the primary score of a table is the RMSE when the table
has numerical cells and the misclassification rate otherwise. The pool comprises a column
mean/mode imputer (the reference baseline), MICE \citep{van1999flexible}, deep imputers
\citep{yoon2018gain, jarrett2022hyperimpute, du2024remasker, zhang2025diffputer}, the two
per-table generators whose papers give an imputation procedure \citep{zhang2025tabnat,
jolicoeur2024generating}, and in-context imputers that fill each missing cell from the observed ones
\citep{grinsztajn2026tabpfn, qu2026tabiclv2, zhang2025limix}; methods whose paper tunes are also run
tuned (marked (t)). 
TabNAT does not accept missing inputs, so we train it on the table with its missing cells filled by the mean/mode imputer and then impute them with its official procedure.
\ourmethod{} takes the incomplete table as context, with its missing cells in the
missing-in-context state, and answers every row with one conditional query over its missing
cells (Appendix~\ref{app:inference}). 

\paragraph{Class-conditional generation.}
The condition of a table is its training-split label vector, so the synthetic table has
$n_{\mathrm{train}}$ rows with the real class proportions. 
Scores are utility, class-wise $\alpha$-precision and $\beta$-recall, and label consistency, i.e., the fraction of synthetic rows that the scoring predictors assign to the requested class, relative to held-out real rows.
Label agreement, the fraction of rows whose emitted label equals the requested class, is $1$ for every method except CTGAN, whose condition is a soft constraint; we score CTGAN's output as emitted and do not apply rejection sampling, which would be a procedure outside the model. Table~\ref{tab:cond} lists how each method fixes the condition.

\begin{table}[ht]
\caption{How each method conditions on the class in class-conditional generation.}
\label{tab:cond}
\begin{center}
\begin{small}
\setlength{\tabcolsep}{4pt}
\begin{tabular}{@{}lll@{}}
\toprule
Mechanism & Methods & Where the conditional comes from \\
\midrule
Pretrained conditional & \ourmethod{} & label cell observed; the conditional is a training term \\
Trained conditional, per table & TabNAT & masked training over random orders on the table \\
Class input at training & TabDDPM & the label is an input of the denoiser; not generated \\
One model per class & ForestDiffusion & a separate model is trained for each class \\
Discrete-column condition & CTGAN & condition vector plus a cross-entropy penalty (soft) \\
Prompt & GReaT & the label appears in the prompt; permutation training \\
Classifier energy & TabEBM, TabPFGen & Langevin sampling of a pretrained classifier's energy \\
Chained given the label & in-context generators & label observed; the columns are chained \\
\bottomrule
\end{tabular}
\end{small}
\end{center}
\end{table}

\paragraph{Elo.}
We follow the official TabArena procedure: every (table, metric) pair is one task, a method's score on a task is its rank among the entered methods, and Elo ratings are fit to the pairwise win rates with $200$ bootstrap resamples over tasks for the $95\%$ intervals. Generation tasks are table $\times$ \{utility, $\alpha$-precision, $\beta$-recall, authenticity\}; imputation tasks are table $\times$ \{RMSE on numerical cells, error rate on categorical cells, downstream utility\}, where a table enters only the metrics its cells admit; class-conditional tasks are table $\times$ \{utility, class-macro $\alpha$-precision, class-macro $\beta$-recall, label-consistency gap\}. 
A method that has not been run on a table receives the reference baseline's score on that table (column resampling or the mean/mode imputer), or, in class-conditional generation, a score just below the worst method, and is marked $\dagger$.

\paragraph{Excluded methods.}
TabSyn and TabDiff are not entered for imputation or class-conditional generation, because their papers claim neither task in the form we
evaluate: TabSyn's paper evaluates unconditional generation and imputation of a target column by
inpainting, and TabDiff's imputation experiments impute the target column only and require an
auxiliary model trained on that column. TVAE, STaSy, and the reference baselines have no
conditioning mechanism. CTSyn and TabFORGE~\citep{lin2025ctsyn, jiang2026tabular} have no public
code at the time of writing and are discussed in \S\ref{sec:related-diffusion} only.

\section{Generation and imputation with in-context predictors}
\label{app:posthoc}

The in-context generators and imputers of Appendix~\ref{app:protocol} are pretrained on one cell
at a time: a label cell~\citep{grinsztajn2026tabpfn, qu2026tabiclv2} or masked cells scored
separately given the unmasked ones~\citep{zhang2025limix}. Their official procedures obtain a row
or a set of cells from these single-cell conditionals. i) \emph{Generation} chains the predictor
over the columns, $p(x_1 \mid D)\, p(x_2 \mid x_1, D) \cdots$, in one or several column orders and
averages over the orders; the chain is a valid joint for each order, but each order is a
different joint, and the conditionals were never trained together. LimiX follows the chain with
rounds of masking a fraction of the generated cells and refilling them. ii) \emph{Imputation} fills the missing cells of a row from the observed cells, column by column with the cells already filled entering the conditioning (TabPFN-3, TabICLv2) or all at once (LimiX), and returns one value per cell, i.e., a point estimate at the mode or the mean of each conditional rather than a sample of the joint; Appendix~\ref{app:ppd} measures the distance of such point estimates from the joint. iii) The masked objective
$-\log q_\theta(x_\pi \mid x_{-\pi}, D)$ of LimiX is written over a masked set $\pi$, but the
network outputs one distribution per masked cell given the unmasked cells, so for $|\pi| > 1$
the objective is the pseudo-likelihood-style product $\prod_{j \in \pi} q_\theta(x_j \mid
x_{-\pi}, D)$, which parameterizes no joint over $\pi$. In \ourmethod{}, by contrast, the cells of
$S$ are coupled at every reverse step and the objective is a bound on the likelihood of $x_S$
itself (Theorem~\ref{thm:elbo}).
\section{Distance to the posterior predictive: protocol}
\label{app:ppd}

Corollary~\ref{cor:ppd} and Corollary~\ref{cor:min} are statements about the pretraining prior, so we test them as \citet{muller2022transformers} test theirs, on priors whose posterior predictive is available in closed form. Two experiments use these priors differently. Fig.~\ref{fig:ppd} pretrains a small model on the priors themselves and follows its checkpoints. Fig.~\ref{fig:ppd-real} applies the pretrained \ourmethod{} of \S\ref{sec:method-pretraining}, which has never seen these priors, to the Dirichlet prior, so its distances include a prior-mismatch term that the corollaries do not bound.

\paragraph{Priors.}
$x$ denotes a numerical and $c$ a categorical column. i) \emph{GP}: $x_1, x_2 \sim \mathcal{U}(-2, 2)$ and $x_3 = f(x_1, x_2) + \varepsilon$ with $f \sim \mathcal{GP}(0, k)$, $k$ the RBF kernel with output scale $1$, and $\varepsilon \sim \mathcal{N}(0, 0.1^2)$. The posterior predictive of $x_3$ given the other cells and the context is the GP posterior; the length scale is fixed to $0.6$, so that the posterior predictive is exact. ii) \emph{Dirichlet}: four categorical columns $c_1, \dots, c_4$ with three classes each, one table-level PMF $\theta \sim \mathrm{Dir}(0.3 \cdot \mathbf{1})$ over the $81$ joint cells, and rows drawn i.i.d.\ from $\theta$; the posterior predictive of any subset of cells given the rest is Dirichlet--multinomial. iii) \emph{GP+Dirichlet}: one categorical column $c$ with three classes and class proportions $\pi \sim \mathrm{Dir}(0.3 \cdot \mathbf{1})$, two numerical columns as in i), and $x_3 \sim \mathcal{GP}(0, k)$ with $k = \mathrm{RBF}(x) + \mathbf{1}[c = c']$, so that rows of the same class are correlated. The posterior of $x_3$ given $(x_1, x_2, c)$ is the GP posterior with the extended kernel, and that of $c$ given $(x_1, x_2, x_3)$ follows by enumerating the three classes.

\paragraph{Small model (Fig.~\ref{fig:ppd}).}
The model is a $6$-block, width-$256$ instance of the architecture of Appendix~\ref{app:pretraining} ($6.7$M parameters), trained with the Stage-1 recipe on a mixture of the three priors with equal probability, context length log-uniform in $[8, 1{,}024]$, $16$ query rows, and the target pattern of Appendix~\ref{app:pretraining}, for $60$k steps on one GPU; numerical cells are standardized with the mean and the standard deviation of the observed context cells. We evaluate the checkpoints at $10$k, $20$k, $30$k, and $60$k steps on $20$ tables per prior, $16$ query rows per table, context sizes $n \in \{8, \dots, 2{,}048\}$, and $M = 512$ samples per query. The score is the negative log-likelihood of the true value of the hidden cells under the samples: the empirical frequency of the true cell among the samples for categorical targets and a Gaussian kernel density of the samples for a numerical target. The exact posterior predictive is scored in closed form, and the floor applies the same estimator to $M$ of its own samples.

\paragraph{Target sets (Fig.~\ref{fig:ppd-real}).}
Each query hides a set $S$ of cells of a row of the Dirichlet prior, $|S| \in \{1, \dots, 4\}$, with $n = 1{,}024$ context rows, $20$ tables, $16$ query rows per table, and $M = 1{,}024$ samples per query. The distance to the posterior predictive is the total variation to the exact joint PMF. The figure  carries the \emph{Monte-Carlo floor}, the same statistic between the exact posterior predictive and $M$ of its own samples, and the \emph{factorized oracle}, the product of the exact per-cell marginals, which is the best any method that fills the cells of $S$ independently can reach. The baselines sample the same cells with their own generation procedures, conditioned on the observed cells: TabNAT is trained on the context table and samples the hidden cells autoregressively in a random order; TabPFN-3 samples them column by column from its predictive distribution at temperature $1$, averaging $10$ permutations of the conditioning columns per column; TabICLv2 runs the column-wise chain of its official generation procedure over the hidden columns in the original order at temperature $1$; LimiX runs the autoregressive stage of its official procedure over the hidden columns and then three rounds of masking and refilling within $S$. Their imputation procedures, which return one value per cell, are farther from the posterior predictive than the factorized oracle at every $|S|$ and are omitted from the figure.

\section{Target-set scenarios: protocol}
\label{app:targets}

Unlike the imputation benchmark of Appendix~\ref{app:protocol}, which injects missingness into the training split and fills that split, the scenarios of \S\ref{sec:exp-targets} keep the training split as it is, natural missingness included, and hide cells of the \emph{test} rows, so that the context is fixed and only the target set varies; the hidden values are known and scored. Hidden cells are drawn by the same generator as the benchmark (MCAR at rates $10$--$70\%$, MAR with logistic drivers among the numerical features, MNAR by self-masking; \citealt{zhang2025diffputer}), or by hiding $|S| \in \{1, 2, 4\}$ whole feature columns, chosen at random per table, in every test row. Scores are the per-cell errors of the benchmark, divided by the errors of the mean/mode imputer under the same mask; values below $1$ improve on it. The eight baselines are MICE, GAIN, DiffPuter, ForestDiffusion, and TabNAT, each fit on the training split and applied to the test rows, and the in-context imputers TabPFN-3, TabICLv2, and LimiX with the training split as context. \ourmethod{} runs Algorithm~\ref{alg:inference} unchanged over the whole grid with the training split as context and the imputation settings of Table~\ref{tab:app-tasks}. All results use one mask seed. TabPFN-3, ForestDiffusion, and DiffPuter did not finish the larger tables within our compute budget, so Table~\ref{tab:targets} reports the $24$ tables on which all nine methods completed, and Table~\ref{tab:targets-full} reports every method on those tables together with the number of tables each method has results on.

\section{What the second stage adds}
\label{app:stage}

Stage~2 introduces long tables, the context-missingness distribution, and the multi-corruption estimator (Appendix~\ref{app:pretraining}). We compare the final Stage-1 checkpoint with the final model on the three tasks of Appendix~\ref{app:protocol}, with the same inference settings. The Stage-1 model has no state for a missing context cell, so on the $11$ tables whose context has natural missingness its context is pre-filled (mean/mode for imputation; median and an explicit missing category for generation), whereas the final model reads the incomplete context directly; on these tables the two models therefore see different inputs, which we mark in Fig.~\ref{fig:stage} and exclude from the size analysis of generation and class-conditional generation. Imputation error is lower after Stage~2 on $45$ of $48$ tables (mean $0.129$ against $0.156$) and on every size group. Generation utility improves on $40$ of $48$ tables and class-conditional utility on $26$ of $36$, and both gains concentrate on tables above $10$k rows ($+0.04$ on $9$ of $9$ and $4$ of $4$ tables), i.e., the lengths that Stage~1 never saw.

\section{Additional results}
\label{app:results}

\paragraph{Full results.}
Tables~\ref{tab:full-gen}, \ref{tab:full-imp}, and \ref{tab:full-cond} report every method on generation, imputation, and class-conditional generation over the $48$ and $36$ tables, including the reference baselines and the in-context methods.
\begin{table}[ht]
\caption{Generation on the $48$ tables: Elo with $95\%$ bootstrap intervals and per-metric means over the tables each method completed ($n$). Utility is relative to training on real rows; $\alpha$-precision, $\beta$-recall, and authenticity follow \citet{alaa2022faithful}; W$_1$ (median over tables) and TV are column-level distances and $|\Delta\rho|$ the mean absolute difference of pairwise correlations. Time is the median wall-clock per table in seconds (training plus sampling for per-table models; inference for the others), measured on one GPU. The Elo tasks are table $\times$ \{utility, $\alpha$-precision, $\beta$-recall, authenticity\}. (t) = run with the hyperparameters tuned in the original paper, which provides no default. $\dagger$: not completed on every table; missing tables are filled with column resampling in the Elo.}
\label{tab:full-gen}
\begin{center}
\begin{footnotesize}
\resizebox{\linewidth}{!}{%
\begin{tabular}{@{}lrlrrrrrrrr@{}}
\toprule
Method & $n$ & Elo & Utility $\uparrow$ & $\alpha$-prec.\ $\uparrow$ & $\beta$-rec.\ $\uparrow$ & Auth.\ $\uparrow$ & W$_1$ $\downarrow$ & TV $\downarrow$ & $|\Delta\rho|$ $\downarrow$ & Time (s) \\
\midrule
\ourmethod{} & 48 & 1099 [1063, 1126] & 0.914 & 0.790 & 0.408 & 0.569 & 0.054 & 0.038 & 0.038 & 25 \\
TabNAT & 48 & 1093 [1057, 1121] & 0.899 & 0.805 & 0.421 & 0.559 & 0.046 & 0.023 & 0.030 & 1696 \\
TabSyn & 48 & 1060 [1030, 1079] & 0.864 & 0.738 & 0.385 & 0.590 & 0.059 & 0.055 & 0.037 & 369 \\
TabDiff (t) & 48 & 1059 [1031, 1086] & 0.886 & 0.730 & 0.380 & 0.591 & 0.060 & 0.022 & 0.042 & 2119 \\
TabPFN-3 (chained) & 48 & 1055 [1027, 1088] & 0.891 & 0.662 & 0.353 & 0.632 & 0.079 & 0.081 & 0.059 & 571 \\
ForestDiffusion$^\dagger$ & 42 & 1051 [1023, 1081] & 0.886 & 0.770 & 0.409 & 0.568 & 0.059 & 0.061 & 0.030 & 1181 \\
TabICLv2 (chained) & 48 & 1016 [992, 1042] & 0.880 & 0.590 & 0.304 & 0.648 & 0.169 & 0.062 & 0.081 & 72 \\
TabDDPM (t)$^\dagger$ & 45 & 1004 [968, 1046] & 0.818 & 0.660 & 0.371 & 0.618 & 0.263 & 0.095 & 0.137 & 241 \\
GReaT$^\dagger$ & 41 & 1001 [980, 1024] & 0.795 & 0.767 & 0.391 & 0.594 & 0.274 & 0.169 & 0.109 & 8106 \\
TVAE (t)$^\dagger$ & 47 & 996 [978, 1016] & 0.795 & 0.670 & 0.320 & 0.628 & 0.123 & 0.123 & 0.081 & 529 \\
TabEBM$^\dagger$ & 32 & 988 [959, 1018] & 0.920 & 0.641 & 0.408 & 0.536 & 0.050 & 0.059 & 0.019 & 183 \\
TabPFGen$^\dagger$ & 31 & 964 [935, 994] & 0.851 & 0.647 & 0.405 & 0.549 & 0.077 & 0.059 & 0.042 & 131 \\
STaSy (t)$^\dagger$ & 47 & 961 [928, 989] & 0.795 & 0.530 & 0.264 & 0.698 & 0.192 & 0.138 & 0.072 & 1871 \\
Gaussian mixture$^\dagger$ & 47 & 945 [921, 970] & 0.777 & 0.521 & 0.251 & 0.689 & 0.179 & 0.056 & 0.028 & 11 \\
CTGAN (t) & 48 & 937 [912, 960] & 0.727 & 0.480 & 0.226 & 0.714 & 0.161 & 0.113 & 0.113 & 306 \\
LimiX (chained)$^\dagger$ & 46 & 889 [846, 928] & 0.649 & 0.345 & 0.133 & 0.830 & 0.683 & 0.277 & 0.153 & 219 \\
Column resampling & 48 & 886 [856, 918] & 0.518 & 0.425 & 0.197 & 0.743 & 0.017 & 0.017 & 0.203 & 0 \\
\bottomrule
\end{tabular}}
\end{footnotesize}
\end{center}
\end{table}

\begin{table}[ht]
\caption{Imputation on the $48$ tables ($30\%$ of the feature cells hidden completely at random; $3$ mask seeds). The Elo tasks are table $\times$ \{RMSE on numerical cells, error rate on categorical cells, downstream utility\} (125 tasks); the second Elo uses only the 41 tables every method completed (107 tasks). Error is the primary score of a table (min-max--scaled RMSE where the table has numerical cells, misclassification rate otherwise), averaged over the tables each method completed ($n$); rank is the mean rank over tables. RMSE and error rate are also reported separately over the tables with numerical and with categorical cells (counts in parentheses); cells with RMSE above $5$ (out-of-range outputs) are excluded from the means and counted. Downstream is the utility of a predictor trained on the imputed table relative to one trained on the complete table. (t) = tuned. $\dagger$: not completed on every table; missing tables are filled with the mean/mode imputer in the first Elo.}
\label{tab:full-imp}
\begin{center}
\begin{footnotesize}
\resizebox{\linewidth}{!}{%
\begin{tabular}{@{}lrllrrrrr@{}}
\toprule
Method & $n$ & Elo & Elo (common) & Error $\downarrow$ & Rank $\downarrow$ & RMSE $\downarrow$ (num.) & Error rate $\downarrow$ (cat.) & Downstream $\uparrow$ \\
\midrule
\ourmethod{} & 48 & 1458 [1389, 1521] & 1458 [1382, 1539] & 0.129 & 1.88 & 0.108 (44) & 0.314 (33) & 0.962 \\
TabICLv2 (per cell) & 48 & 1201 [1168, 1238] & 1213 [1167, 1258] & 0.146 & 4.38 & 0.126 (44) & 0.375 (33) & 0.839 \\
ReMasker (t) & 48 & 1108 [1071, 1145] & 1113 [1075, 1159] & 0.157 & 3.58 & 0.121 (44) & 0.467 (33) & 0.829 \\
HyperImpute & 48 & 1073 [1040, 1100] & 1071 [1043, 1103] & 0.163 & 5.19 & 0.129 (44) & 0.473 (33) & 0.828 \\
ReMasker & 48 & 1066 [1034, 1099] & 1072 [1033, 1114] & 0.167 & 4.90 & 0.133 (44) & 0.483 (33) & 0.827 \\
DiffPuter & 48 & 1063 [1019, 1104] & 1086 [1045, 1125] & 0.204 (3 excl.) & 7.31 & 0.170 (44) & 0.466 (33) & 0.907 \\
ForestDiffusion$^\dagger$ & 45 & 1043 [1008, 1087] & 1037 [1003, 1078] & 0.180 & 8.80 & 0.149 (41) & 0.428 (31) & 0.923 \\
TabNAT & 48 & 1029 [992, 1076] & 1023 [978, 1061] & 0.177 & 9.08 & 0.152 (44) & 0.413 (33) & 0.910 \\
Mean/mode & 48 & 992 [950, 1034] & 994 [942, 1046] & 0.195 & 10.96 & 0.165 (44) & 0.443 (33) & 0.904 \\
TabPFN-3 (per cell)$^\dagger$ & 42 & 973 [937, 1010] & 976 [931, 1018] & 0.203 & 10.62 & 0.175 (39) & 0.440 (30) & 0.813 \\
LimiX (per cell) & 48 & 957 [913, 987] & 952 [914, 993] & 0.175 & 6.85 & 0.140 (44) & 0.499 (33) & 0.816 \\
GAIN (t) & 48 & 821 [789, 856] & 823 [783, 860] & 0.189 & 10.33 & 0.154 (44) & 0.522 (33) & 0.815 \\
MICE (t) & 48 & 815 [774, 847] & 806 [774, 840] & 0.194 (1 excl.) & 10.54 & 0.158 (44) & 0.530 (33) & 0.817 \\
GAIN & 48 & 777 [745, 808] & 776 [737, 811] & 0.192 & 11.08 & 0.157 (44) & 0.527 (33) & 0.815 \\
MICE & 48 & 631 [580, 678] & 597 [533, 656] & 0.281 & 13.60 & 0.252 (44) & 0.555 (33) & 0.810 \\
\bottomrule
\end{tabular}}
\end{footnotesize}
\end{center}
\end{table}

\begin{table}[ht]
\caption{Class-conditional generation on the $36$ classification tables: Elo with $95\%$ bootstrap intervals and means over the tables each method completed ($n$). The Elo tasks are table $\times$ \{utility, class-macro $\alpha$-precision, class-macro $\beta$-recall, label-consistency gap\}, where the gap is the absolute difference between the label consistency of the synthetic rows and that of held-out real rows. Label consistency is the fraction of synthetic rows assigned to the requested class by the scoring predictors trained on real rows, divided by the same fraction on held-out real rows ($1$ = as separable as real data; above $1$ = classes separated more than in the real data); macro is the unnormalized macro-average over classes; agreement is the fraction of rows whose emitted label equals the requested class. TV and W$_1$ (median) are per-class distances, macro-averaged. (t) = the generation model with the hyperparameters tuned in its original paper. $\dagger$: not completed on every table; a missing table is scored just below the worst method on that table in the Elo.}
\label{tab:full-cond}
\begin{center}
\begin{footnotesize}
\resizebox{\linewidth}{!}{%
\begin{tabular}{@{}lrlrrrrrrrr@{}}
\toprule
Method & $n$ & Elo & Utility $\uparrow$ & $\alpha$-prec.\ $\uparrow$ & $\beta$-rec.\ $\uparrow$ & Consist.\ (rel.) & Consist.\ (macro) & Agreement & TV $\downarrow$ & W$_1$ $\downarrow$ \\
\midrule
\ourmethod{} & 36 & 1163 [1135, 1191] & 0.926 & 0.624 & 0.349 & 0.983 & 0.675 & 1.00 & 0.064 & 0.096 \\
TabPFN-3 (chained) & 36 & 1144 [1116, 1170] & 0.917 & 0.591 & 0.349 & 0.986 & 0.676 & 1.00 & 0.086 & 0.103 \\
TabNAT & 36 & 1122 [1091, 1157] & 0.881 & 0.675 & 0.381 & 0.972 & 0.671 & 1.00 & 0.048 & 0.117 \\
ForestDiffusion$^\dagger$ & 32 & 1076 [1040, 1109] & 0.908 & 0.638 & 0.386 & 1.016 & 0.738 & 1.00 & 0.081 & 0.093 \\
TabEBM$^\dagger$ & 32 & 1053 [1015, 1089] & 0.922 & 0.526 & 0.363 & 1.008 & 0.723 & 1.00 & 0.078 & 0.073 \\
TabICLv2 (chained) & 36 & 1019 [991, 1048] & 0.911 & 0.449 & 0.267 & 0.941 & 0.665 & 1.00 & 0.084 & 0.367 \\
TabDDPM (t)$^\dagger$ & 33 & 989 [953, 1033] & 0.826 & 0.570 & 0.345 & 0.909 & 0.669 & 1.00 & 0.124 & 0.546 \\
GReaT$^\dagger$ & 30 & 964 [932, 999] & 0.849 & 0.656 & 0.376 & 0.953 & 0.656 & 1.00 & 0.168 & 0.271 \\
TabPFGen$^\dagger$ & 31 & 951 [909, 990] & 0.851 & 0.507 & 0.339 & 1.093 & 0.818 & 1.00 & 0.088 & 0.176 \\
CTGAN (t)$^\dagger$ & 35 & 775 [732, 810] & 0.597 & 0.415 & 0.206 & 0.794 & 0.466 & 0.70 & 0.172 & 0.292 \\
LimiX (chained)$^\dagger$ & 34 & 743 [708, 777] & 0.729 & 0.272 & 0.122 & 0.881 & 0.573 & 1.00 & 0.292 & 0.805 \\
\midrule
Real rows (held out) & 36 & -- & -- & -- & -- & 1.000 & 0.709 & -- & -- & -- \\
\bottomrule
\end{tabular}}
\end{footnotesize}
\end{center}
\end{table}

\clearpage
\paragraph{Target-set scenarios, all methods.}
Table~\ref{tab:targets-full} extends Table~\ref{tab:targets} to the nine methods.

\begin{table}[h]
\caption{Per-cell error on held-out rows relative to a mean/mode imputer (median over tables; below $1$ = better than mean/mode), as the target set varies: the MCAR rate, the mechanism at $30\%$, and the number $|S|$ of whole hidden columns. $24$ tables on which every method completed (the $15$ smallest and $9$ medium-sized ones; $23$ for $|S|{=}4$). Best per row in bold.}
\label{tab:targets-full}
\begin{center}
\begin{footnotesize}
\resizebox{\linewidth}{!}{%
\begin{tabular}{@{}llccccccccc@{}}
\toprule
 & & \ourmethod{} & TabNAT & ForestDiff. & DiffPuter & GAIN & MICE & TabPFN-3 & TabICLv2 & LimiX \\
\midrule
\multirow{4}{*}{MCAR rate} & $10\%$ & \textbf{0.622} & 0.777 & 1.000 & 0.901 & 1.001 & 1.157 & 0.912 & 0.729 & 0.872 \\
 & $30\%$ & \textbf{0.675} & 0.860 & 1.082 & 0.966 & 1.028 & 1.233 & 1.098 & 0.812 & 0.916 \\
 & $50\%$ & \textbf{0.770} & 0.932 & 1.113 & 0.955 & 1.141 & 1.237 & 1.212 & 0.921 & 0.966 \\
 & $70\%$ & \textbf{0.843} & 0.993 & 1.182 & 1.005 & 1.290 & 1.296 & 1.328 & 0.974 & 1.016 \\
\midrule
\multirow{2}{*}{$30\%$} & MAR & \textbf{0.628} & 0.887 & 1.065 & 0.887 & 1.064 & 1.212 & 1.019 & 0.807 & 0.931 \\
 & MNAR & \textbf{0.704} & 0.871 & 1.012 & 0.866 & 1.063 & 1.143 & 1.068 & 0.824 & 0.910 \\
\midrule
\multirow{3}{*}{Columns} & $|S|{=}1$ & \textbf{0.523} & 0.780 & 1.019 & 0.826 & 1.010 & 1.193 & 0.785 & 0.575 & 1.551 \\
 & $|S|{=}2$ & \textbf{0.600} & 0.895 & 1.053 & 0.822 & 1.001 & 1.158 & 0.992 & 0.776 & 1.702 \\
 & $|S|{=}4$ & \textbf{0.767} & 0.960 & 1.054 & 0.912 & 1.166 & 1.235 & 1.054 & 0.888 & 1.655 \\
\midrule
\multicolumn{2}{@{}l}{Tables with results (of $48$)} & 48 & 48 & 36 & 42 & 48 & 48 & 28 & 48 & 48 \\
\bottomrule
\end{tabular}}
\end{footnotesize}
\end{center}
\end{table}

\end{document}